\documentclass{article}
\usepackage{atma_arxiv,times}

\usepackage{amsmath,amsfonts,bm}

\def\eqref#1{equation~\ref{#1}}

\def\1{\bm{1}}

\DeclareMathAlphabet{\mathsfit}{\encodingdefault}{\sfdefault}{m}{sl}
\SetMathAlphabet{\mathsfit}{bold}{\encodingdefault}{\sfdefault}{bx}{n}

\usepackage{microtype}
\usepackage{hyperref}
\usepackage{url}
\usepackage{booktabs}
\usepackage{amsmath}
\usepackage{amssymb}
\usepackage{amsfonts}
\usepackage{amsthm}
\usepackage{graphicx}
\usepackage{array}
\usepackage{longtable}
\usepackage{pdflscape}
\usepackage{listings}
\usepackage{placeins}
\graphicspath{{./}}

\definecolor{darkblue}{rgb}{0,0,0.5}
\hypersetup{colorlinks=true,citecolor=darkblue,linkcolor=darkblue,urlcolor=darkblue}
\newtheorem{assumption}{Assumption}
\newtheorem{proposition}{Proposition}

\title{ATMA: Long-Context Language Modeling via Polar Attention and Gated-Delta Compression Memory}
\author{Habibullah Akbar \\
Kreasof AI \\
Jakarta, Indonesia \\
\texttt{habibullah.akbar@kreasof.my.id}}

\iclrfinalcopy
\begin{document}
\maketitle

\begin{abstract}
Length extrapolation in language models involves competing objectives: retrieval fidelity, long-document likelihood, short-context quality, and inference cost. We present \textbf{ATMA}, a 378M-parameter hybrid recipe that combines Polar Attention with gated-delta recurrent memory, and study these objectives as a Pareto problem rather than claiming general architectural dominance. Polar Attention separates a normalized direction channel from a bounded participation-ratio magnitude channel. We select the recipe with a complete 120-cell, 1B-token factorial sweep, then train matched NoPE, RoPE, and Polar variants for 9.816B tokens at length 2K and evaluate them through 256K. Across the factorial, memory improves Polar's 64K retrieval score in all 20 matched cells (mean $+47.8$ points), whereas its effect on NoPE is small and inconsistent. At 256K, Polar retains 34.4\% teacher-forced target-token accuracy and 9.0\% exact five-token accuracy; exact retrieval is 18.0\% on synthetic contexts but 0.0\% on FinePDFs contexts. Polar also limits mean fixed-target bits-per-byte degradation to $1.26\times$, at a 1.9-point mean cost on eight short-context tasks. Raven baselines lead BABILong and have length-independent decode state, illustrating a different point on the frontier. Finally, a post-hoc checkpoint audit shows that nearly identical 2K validation curves can conceal a 6.70-nat difference at 256K. Because those runs were neither seed-paired nor randomized across devices, we interpret this as checkpoint variability associated with an infrastructure transition, not a causal hardware effect. Code: \url{https://github.com/kreasof-ai/atma}
\end{abstract}

\section{Introduction}

Large language models increasingly process code repositories, books, legal documents, and scientific papers that exceed their pretraining sequence length. Full softmax attention preserves token-level access in principle, but its probability mass can disperse as the key set grows~\citep{nakanishi2025scalable,velickovic2025softmax,barbero2024glasses}. Sliding windows bound cost by removing older tokens from direct access. Recurrent sequence models instead keep fixed-size state but compress history. Long-context modeling therefore exposes a multi-objective trade-off among retrieval fidelity, sequence likelihood, reasoning, short-context quality, and serving cost.

We present \textbf{ATMA}, a hybrid architecture that combines short gated convolutions with four \textbf{Polar Attention} layers as global mixers. Polar Attention represents ``what matched'' with a normalized direction and ``how much matched'' with a bounded participation-ratio channel. A learned null sink follows the extreme-value scale of irrelevant scores. Each global layer also contains gated-delta memory, providing a compressed recurrent history alongside full-context attention (Appendix~\ref{appendix:theory}).

\begin{figure}[t]
\centering
\includegraphics[width=\linewidth]{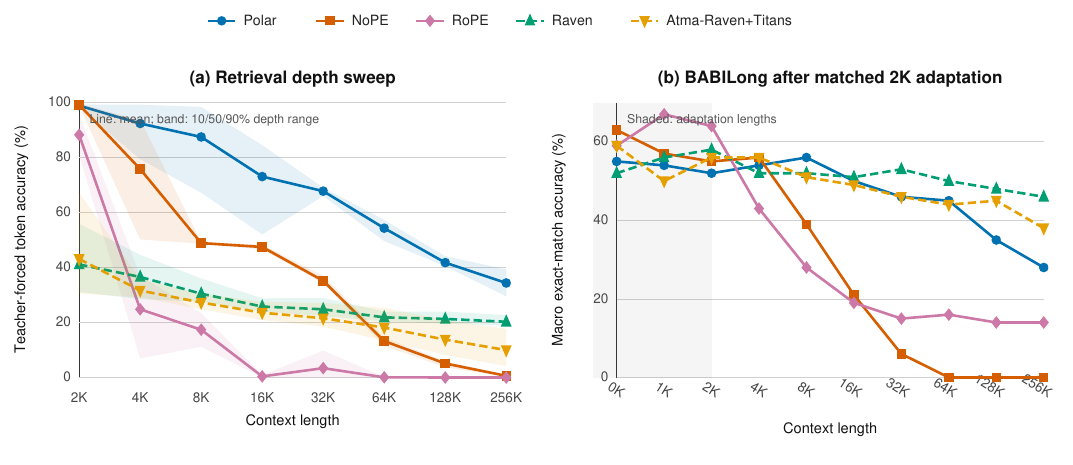}
\caption{Results after 9.816B-token pretraining. \textbf{Left:} teacher-forced target-token retrieval accuracy, averaged over tasks, haystack suites, and needle depths. \textbf{Right:} BABILong exact match after adaptation. Polar retains more retrieval signal than the matched attention variants, whereas the separately optimized Raven family leads long-context BABILong.}
\label{fig:main_results}
\end{figure}

Our contributions are threefold. First, a complete 120-cell, 1B-token factorial sweep measures interactions among attention core, recurrent memory, training-time local windows, distractor alignment, and representation regularization. Second, we promote one recipe to matched 9.816B-token NoPE, RoPE, and Polar runs and evaluate retrieval, BABILong, fixed-target likelihood, short-context tasks, and serving. Raven models provide a strong but separately optimized recurrent-family comparison~\citep{raven2026}. Figure~\ref{fig:main_results} summarizes the resulting frontier: Polar improves long-context retrieval and likelihood among the matched attention variants, while Raven favors BABILong and decode efficiency. Third, a post-hoc audit demonstrates that short-context convergence alone did not predict extrapolation for three NoPE checkpoints. This audit motivates reporting extended-context behavior and execution metadata, but does not identify the device transition as causal.

\section{Related Work}

Softmax attention~\citep{Vaswani+2017} has quadratic compute and can lose selectivity as the number of keys grows. RoPE~\citep{rope2021} and ALiBi~\citep{alibi2021} alter positional information while retaining dense simplex normalization. Scalable-Softmax, adaptive-temperature analyses, and sparse alternatives address dispersion or length-dependent sharpness~\citep{nakanishi2025scalable,velickovic2025softmax,vasylenko2026sparse}. Threshold Differential Attention (TDA) uses a $\sqrt{\log n}$ threshold motivated by maxima of noise scores~\citep{huang2026threshold}. Polar combines length-dependent temperature and extreme-value scaling with a differentiable null sink and bounded magnitude channel. We compare Polar with matched NoPE and RoPE controls; comparisons with these other methods are conceptual rather than empirical.

Linear attention, state-space models, and gated recurrences exchange full-resolution token access for compressed fixed-size history~\citep{mamba2023,deltanet2024,yang2024gated,titans2025}. Raven combines routed sparse memory, decay dynamics, and gated-delta recurrence~\citep{raven2026}. We treat it as a strong sub-quadratic baseline, but not as an attention ablation because it uses a different model family and optimizer. Hybrid models such as LFM2 combine local and global mixers for efficiency~\citep{lfm2_2025,physics_lms_2025}. ATMA follows this hybrid pattern while changing the global mixer and adding a recurrent memory stream.

\section{Method}

\subsection{Hybrid architecture}

ATMA is a 16-layer decoder with 12 LFM2-style gated convolutional layers interleaved with four global attention layers. All blocks use pre-normalization and squared-ReLU MLPs. The attention variants share hidden size 1024, eight query heads, two key-value heads, head dimension 128, tokenizer, data order, optimizer, schedule, and training budget; only the attention core changes among NoPE, RoPE, and Polar. Canon depthwise causal convolutions are applied to $q$, $k$, and $v$ before attention~\citep{physics_lms_2025}.

\subsection{Polar attention}

For query $i$ and its $n_i=i+1$ visible keys, let $\sigma_{ij}=q_i^\top k_j/\sqrt{d_k}$. Each head learns a length temperature and a virtual null logit,
\begin{equation}
 t_i=1+\operatorname{softplus}(\alpha)\log n_i,\qquad
 \nu_i=b+\operatorname{softplus}(\gamma)\sqrt{\log(n_i+1)}.
\label{eq:length}
\end{equation}
The second term tracks the extreme-value scale of unrelated scores. Appending the null as a softmax competitor gives
\begin{equation}
 w_i=\operatorname{softmax}\!\left(t_i[\sigma_{i\bullet},\nu_i]\right),\qquad
 s_i=\sum_jw_{ij}v_j+w_{iN}v_{\rm null}.
\end{equation}
The \emph{direction} channel is $c_i=s_i/\max(\lVert s_i\rVert_2,\varepsilon)$, hence $\lVert c_i\rVert_2\leq1$. To retain bounded information about match multiplicity, real-key weights are renormalized as $\hat w_{ij}=w_{ij}/\sum_kw_{ik}$ and summarized by the participation ratio
\begin{equation}
 n_{{\rm eff},i}=\left(\sum_j\hat w_{ij}^2\right)^{-1},\quad
 m_{{\rm eff},i}=n_{{\rm eff},i}(1-w_{iN}),\quad
 \mu_i=\tanh\!\left(\operatorname{softplus}(\beta)\log(1+m_{{\rm eff},i})\right).
\end{equation}
Thus $\mu_i\in[0,1)$, while the direction channel removes radial scale. This is a channel decomposition, not an independence claim: changing the match set can rotate $c_i$. The Polar output is a gated projection of $c_i$ plus a learned projection of $\mu_i$. These choices were selected through pre-integration falsification tests: additive soft counts accumulated leakage with length, fixed or relative floors were overtaken by noise tails, and an unbounded $\log m$ channel left the training range. The standalone prototype was not retained; Appendix~\ref{appendix:polar_implementation} distinguishes this design provenance from the reproducible integrated oracle and kernel tests. A distractor-alignment auxiliary loss was tested in the factorial but is disabled in the promoted recipe.

\subsection{Gated-delta memory}

Each global layer also maintains a per-head recurrent state $M_t\in\mathbb{R}^{d_v\times d_k}$. With unit-normalized keys and queries, retention $\gamma_t$, and write gate $\beta_t$, the decay-first, self-inclusive recurrence is
\begin{equation}
 M_t=\gamma_tM_{t-1}(I-\beta_tk_tk_t^\top)+\beta_tv_tk_t^\top,\qquad r_t=M_tq_t.
\label{eq:memory}
\end{equation}
RMS-normalized $r_t$ is gated, projected, and added to the Polar output. The projection is zero-initialized. Unit $L_2$ normalization is important: RMS-normalized keys have $\lVert k\rVert_2^2=d_k$ and made the tested delta recurrence diverge, whereas unit keys keep the update eigenvalue along $k$ within the gate-controlled range. Under bounded values and a uniformly contractive retention gate, Appendix~\ref{appendix:theory} gives a sequence-length-independent state bound. Our empirical claim is narrower: the implemented state remains stable in the tested length sweep.

Figure~\ref{fig:machinery} summarizes the two mechanisms and their controlled diagnostics. Appendix~\ref{appendix:reg_modes} defines every regularization mode in the factorial sweep, Appendix~\ref{appendix:polar_implementation} gives the numerical oracle and streaming Triton implementation, and Appendix~\ref{appendix:systems} details the kernel and serving optimizations.

\begin{figure}[t]
\centering
\includegraphics[width=0.88\linewidth]{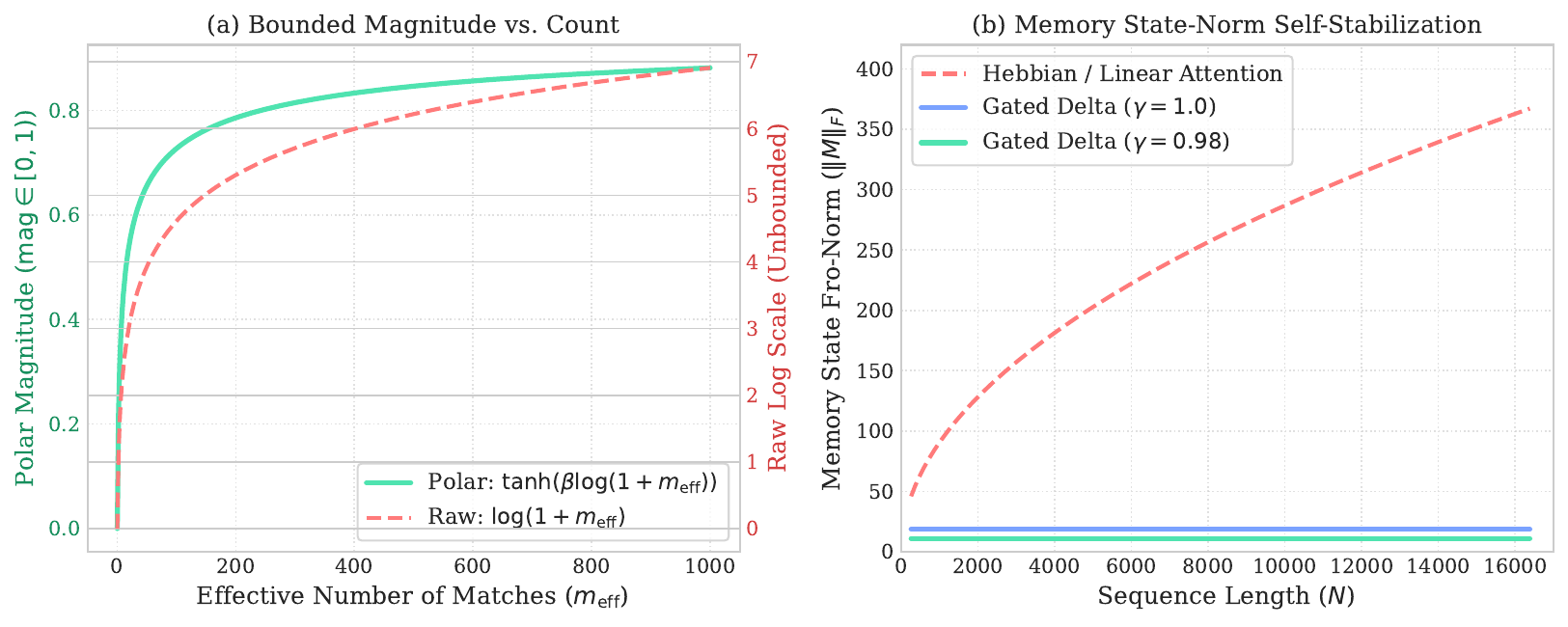}
\caption{ATMA diagnostics. \textbf{Left:} the effective match count is mapped to a bounded magnitude channel. \textbf{Right:} the tested gated-delta state remains flat with length, unlike an ungated Hebbian memory. These controlled diagnostics illustrate the mechanism, they are not a network-level stability proof.}
\label{fig:machinery}
\end{figure}

\section{Experimental Design}

\subsection{Recipe selection}
Stage I trains all $3\times5\times2\times2\times2=120$ combinations of attention type (Polar, RoPE, NoPE), regularization, distractor loss, memory, and a 1024-token training window. Every model receives approximately 1B FineWeb-Edu tokens at length 2048 on an NVIDIA L4, and every evaluation uses full context. These are distinct factorial cells rather than seed replicates; Stage I selects the promoted recipe and is not used for headline absolute comparisons.

\subsection{Matched comparison}
Stage II trains 378.16--378.22M-parameter NoPE, RoPE, and Polar models for 9.816B tokens at length 2048 in the same L40S software and hardware environment. All three use the same Muon/AdamW parameter split. Raven Native (388.54M) and Atma-Raven-Titans (382.37M) use the same data, tokenizer, token budget, and device class, but use AdamW. This follows a 1B-token Raven Native pilot with the ATMA Muon schedule: validation loss improved to 3.76 nats at step 625, then destabilized and ended at 4.62, with 0\% retrieval through 64K. We therefore report the stable Raven models as a separate model-family/optimizer group. The single failed pilot does not establish incompatibility with every Muon configuration.

\paragraph{Evaluation.}
The evaluation contains 24,000 paired retrieval trials over synthetic and FinePDFs haystacks, lengths 2K--256K, depths 10/50/90\%, and 50 trials per cell; BABILong adaptation and exact match~\citep{babilong2024}; fixed-target bits per byte (BPB) on FinePDFs, PG-19, and Proof-Pile; eight likelihood-scored downstream tasks (22,939 examples per model); and single-sequence prefill/decode on one L40S. Retrieval teacher-forces the five-token target and reports both token accuracy and exact five-token accuracy. Neither metric is autoregressive generation success. BABILong uses custom splits and prompts, so its absolute scores are internally comparable rather than leaderboard replications. Protocols, denominators, and complete matrices appear in Appendix~\ref{appendix:protocol}.

\section{Results}

\subsection{Component selection}

Table~\ref{tab:ablation} summarizes matched contrasts from the full factorial rather than selecting individual cells. At 64K, adding memory improves Polar in all 20 combinations of regularizer, distractor setting, and window setting, by 47.8 points on average. The corresponding NoPE effect is small and heterogeneous, and RoPE remains at zero. Among memory-enabled cells, Polar is never worse than NoPE and exceeds RoPE in all 20 comparisons. A training window reduces Polar+memory by 25.5 points on average and is harmful in 8/10 matched cells; distractor alignment has a smaller, nonuniform effect (mean $-6.0$ points; harmful in 6/10). We therefore promote full-context Polar plus memory without distractor alignment. The selected baseline cell reaches 92.5\% token accuracy at 64K; the complete matrix is in Appendix~\ref{appendix:full_ablation}.

\begin{table}[t]
\centering
\small
\setlength{\tabcolsep}{3.4pt}
\resizebox{\linewidth}{!}{%
\begin{tabular}{lrrrr}
\toprule
\textbf{64K matched contrast} & \textbf{Mean $\Delta$} & \textbf{Median $\Delta$} & \textbf{Range} & \textbf{Positive} \\
\midrule
Polar: memory on $-$ off & +47.8 & +51.3 & $[+7.5,+92.5]$ & 20/20 \\
NoPE: memory on $-$ off & +4.6 & +1.9 & $[-15.0,+21.3]$ & 11/20 \\
RoPE: memory on $-$ off & 0.0 & 0.0 & $[0.0,0.0]$ & 0/20 \\
Polar+memory $-$ NoPE+memory & +41.7 & +41.9 & $[0.0,+77.5]$ & 19/20 (+1 tie) \\
Polar+memory $-$ RoPE+memory & +50.3 & +51.9 & $[+18.8,+92.5]$ & 20/20 \\
\bottomrule
\end{tabular}}
\caption{Matched contrasts across the 120-cell, 1B-token factorial. Values are percentage-point changes in teacher-forced five-token accuracy at 64K. Each row holds all other factorial axes fixed. ``Positive'' excludes one numerical tie in the Polar--NoPE comparison.}
\label{tab:ablation}
\end{table}

\subsection{A multi-objective frontier}

Table~\ref{tab:endpoints} reports extrapolation endpoints. At 2K, NoPE and Polar achieve nearly perfect target-token accuracy. At 256K, NoPE and RoPE fall to 0.6\% and 0.0\%, while Polar retains 34.4\% averaged across tasks, suites, and depths. Exact five-token accuracy gives a stricter boundary: Polar reaches 9.0\% overall at 256K, entirely from synthetic haystacks (18.0\% synthetic; 0.0\% FinePDFs). At 64K, Polar obtains 65.7\% exact accuracy on synthetic contexts but 0.0\% on FinePDFs, despite 16.3\% target-token accuracy there. We therefore claim retained target signal and exact synthetic retrieval, not successful real-text retrieval at those lengths. Full token, exact, depth, and haystack matrices are in Appendix~\ref{appendix:retrieval}.

\begin{table}[t]
\centering
\small
\setlength{\tabcolsep}{3.2pt}
\resizebox{\linewidth}{!}{%
\begin{tabular}{llrrrrrrrr}
\toprule
\textbf{Group} & \textbf{Model} & \multicolumn{3}{c}{\textbf{Retrieval (\%)}} & \multicolumn{2}{c}{\textbf{BABILong (\%)}} & \multicolumn{3}{c}{\textbf{BPB}} \\
\cmidrule(lr){3-5}\cmidrule(lr){6-7}\cmidrule(lr){8-10}
& & \textbf{Tok. 2K} & \textbf{Tok. 256K} & \textbf{Exact 256K} & \textbf{2K} & \textbf{256K} & \textbf{2K} & \textbf{256K} & \textbf{ratio} \\
\midrule
Matched & NoPE & 99.2 & 0.6 & 0.0 & 55 & 0 & \textbf{1.432} & 8.097 & 5.65$\times$ \\
Matched & \textbf{Polar} & 98.9 & \textbf{34.4} & \textbf{9.0} & 52 & \textbf{28} & 1.451 & \textbf{1.825} & \textbf{1.26$\times$} \\
Matched & RoPE & 88.2 & 0.0 & 0.0 & \textbf{64} & 14 & 1.439 & 2.512 & 1.75$\times$ \\
\midrule
Raven/AdamW & Atma-Raven-Titans & 43.1 & 10.0 & 0.0 & 56 & 38 & \textbf{1.525} & \textbf{1.551} & 1.02$\times$ \\
Raven/AdamW & \textbf{Raven Native} & 41.1 & \textbf{20.3} & 0.0 & \textbf{58} & \textbf{46} & 1.581 & 1.597 & \textbf{1.01$\times$} \\
\bottomrule
\end{tabular}}
\caption{Long-context endpoints. Retrieval averages tasks, suites, and depths; ``Tok.'' is teacher-forced target-token accuracy and ``Exact'' requires all five target tokens to be correct. BABILong is macro exact match. BPB averages three fixed-target datasets; lower is better. Bold is best within each comparison group.}
\label{tab:endpoints}
\end{table}

\begin{table}[t]
\centering
\small
\setlength{\tabcolsep}{3.2pt}
\resizebox{\linewidth}{!}{%
\begin{tabular}{llrrrrrrrr}
\toprule
\textbf{Model} & \textbf{Haystack} & \textbf{2K} & \textbf{4K} & \textbf{8K} & \textbf{16K} & \textbf{32K} & \textbf{64K} & \textbf{128K} & \textbf{256K} \\
\midrule
Polar & Synthetic & 99.6 & 98.4 & 98.5 & 97.0 & 94.9 & 92.3 & 80.3 & 68.4 \\
& FinePDFs & 98.6 & 86.3 & 77.9 & 49.1 & 40.7 & \textbf{16.3} & 3.2 & 0.3 \\
NoPE & Synthetic & 99.5 & 98.7 & 98.9 & 94.9 & 70.3 & 26.5 & 10.2 & 1.2 \\
& FinePDFs & 99.3 & 52.9 & 0.0 & 0.0 & 0.0 & 0.0 & 0.0 & 0.0 \\
RoPE & Synthetic & 81.7 & 44.9 & 33.1 & 0.8 & 6.8 & 0.1 & 0.0 & 0.0 \\
& FinePDFs & 94.8 & 4.7 & 0.2 & 0.0 & 0.0 & 0.0 & 0.0 & 0.0 \\
Raven Native & Synthetic & 55.9 & 57.7 & 51.4 & 48.3 & 44.3 & 40.9 & 38.9 & 35.1 \\
& FinePDFs & 22.1 & 15.4 & 8.2 & 3.3 & 5.3 & 2.9 & 3.7 & 5.4 \\
\bottomrule
\end{tabular}}
\caption{Zero-shot retrieval accuracy (\%) by haystack type. Values average teacher-forced token accuracy over needle depths. The recurrent hybrid is omitted for space and reported in Appendix~\ref{appendix:retrieval}.}
\label{tab:haystack}
\end{table}

Table~\ref{tab:haystack} separates target-token accuracy by haystack and length; Appendix~\ref{appendix:retrieval} reports the corresponding exact-sequence values. Polar also limits mean BPB degradation to $1.26\times$ at 256K, compared with $5.65\times$ for NoPE and $1.75\times$ for RoPE. This stability accompanies a short-context cost: mean accuracy over eight 2K controls is 43.29\% for Polar, 45.15\% for NoPE, and 44.60\% for RoPE. Figure~\ref{fig:short_context_controls} shows that the gap is task-dependent rather than uniform. The 1.86-point mean gap to NoPE prevents a blanket quality claim. Exact per-task values are in Appendix~\ref{appendix:quality}.

\begin{figure}[t]
\centering
\includegraphics[width=\linewidth]{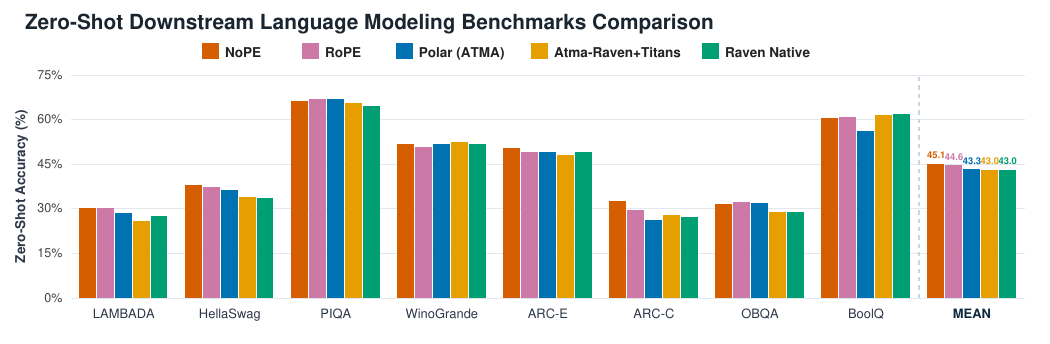}
\caption{Zero-shot accuracy on eight 2K language-modeling controls (22,939 examples per model). The rightmost bars are unweighted task means. Exact values and metric definitions are in Table~\ref{table:base_tasks_breakdown}.}
\label{fig:short_context_controls}
\end{figure}

Table~\ref{tab:systems} makes the quality--systems trade-off explicit. Raven Native reaches 46\% BABILong at 256K and keeps BPB nearly flat, while its fixed recurrent state avoids a sequence-length-dependent KV cache. At 128K, Atma-Raven-Titans decodes at 2.27 ms/token versus 16.3 ms/token for Polar; Polar instead has higher synthetic retrieval and training MFU. The BABILong samples are small: simple Wilson intervals at 256K are 20.1--37.5\% for Polar (28/100) and 36.6--55.7\% for Raven Native (46/100). Because Raven also changes model family and optimizer, these rows locate practical operating points rather than isolate one causal mechanism.

\begin{table}[t]
\centering
\small
\setlength{\tabcolsep}{3.2pt}
\resizebox{\linewidth}{!}{%
\begin{tabular}{llrrrrr}
\toprule
\textbf{Group} & \textbf{Model} & \textbf{Base mean} & \textbf{128K prefill} & \textbf{Decode} & \textbf{Peak alloc.} & \textbf{Train MFU} \\
& & (\%) & (s) & (ms/token) & (GiB) & (\%, 2K) \\
\midrule
Matched & NoPE & 45.1 & 1.47 & 16.4 & 39.40 & 42.77 \\
Matched & Polar & 43.3 & 1.62 & 16.3 & 39.41 & 40.44 \\
Matched & RoPE & 44.6 & 1.47 & 16.5 & 39.41 & 42.45 \\
Raven/AdamW & Atma-Raven-Titans & 43.1 & 1.00 & 2.27 & 8.62 & 30.92 \\
Raven/AdamW & Raven Native & 43.0 & 1.36 & 3.27 & 6.46 & 21.53 \\
\bottomrule
\end{tabular}}
\caption{Short-context and systems trade-offs. Serving uses one sequence and one timing sample on one L40S, so the values are descriptive rather than uncertainty estimates.}
\label{tab:systems}
\end{table}

\FloatBarrier
\section{Checkpoint Variability Audit}

\begin{figure}[!ht]
\centering
\includegraphics[width=\linewidth]{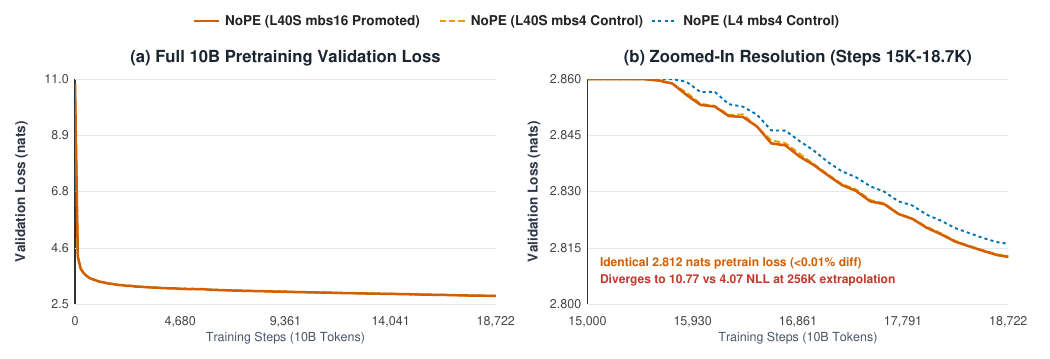}
\caption{NoPE 2K validation trajectories are nearly coincident, yet the checkpoints' 256K NLL values are 10.77, 9.90, and 4.07. The two L40S runs differ in microbatch size; the L4 run differs in device class.}
\label{fig:environment}
\end{figure}

Figure~\ref{fig:environment} presents a post-hoc audit: nearly identical short-context convergence can hide substantially different long-context behavior. Three NoPE runs finish at 2.8126, 2.8128, and 2.8160 validation nats, but their 256K NLLs are respectively 10.77 (L40S, microbatch 16), 9.90 (L40S, microbatch 4), and 4.07 (L4, microbatch 4). The runs were not initialized from a recorded common seed, so this comparison measures checkpoint-to-checkpoint variability, not a device-class treatment effect.

\subsection{Diagnostic journey}
The finding emerged unintentionally during an infrastructure transition rather than a planned device comparison. Stage I and the first 10B NoPE control ran on an L4 with microbatch 4 and produced 4.07 nats at 256K. Moving the promoted runs to L40S with microbatch 16 preserved the 2K trajectory but yielded 10.77 nats at 256K. Re-evaluation under PyTorch 2.12 and 2.13 left both checkpoints unchanged; an L40S microbatch-4 retrain still reached 9.90, and deterministic 2K fixtures showed no gradient or loss discrepancy. These checks exclude evaluation-version mismatch and microbatch size alone in the tested runs. They cannot distinguish random initialization, data-order effects, hardware-dependent kernels, or their interactions. All five headline models were trained in the same L40S environment, so the main model comparison remains environment-matched. Appendix~\ref{appendix:environment} records the audit sequence.

\begin{figure}[!ht]
\centering
\includegraphics[width=\linewidth]{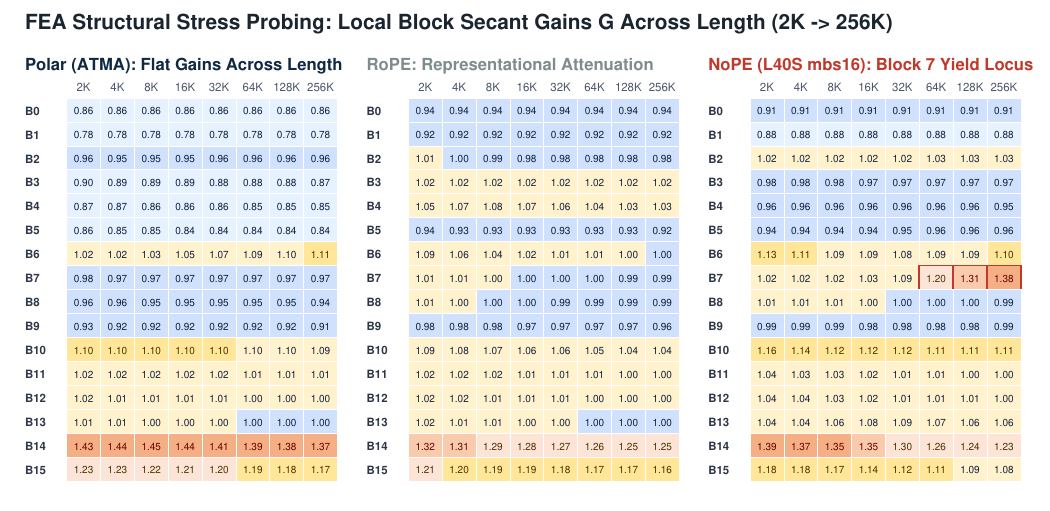}
\caption{Sampled block secant gains across context length. The probe localizes the L40S NoPE checkpoint's change near blocks 6--7, while Polar remains comparatively flat. A sampled gain is a directional diagnostic, not a tight global Lipschitz estimate.}
\label{fig:stress}
\end{figure}
\FloatBarrier

Figure~\ref{fig:stress} localizes the checkpoint difference with a sampled block-level stress probe. The L40S NoPE checkpoint first amplifies the perturbations around blocks 6--7, whereas the L4 checkpoint redistributes them more weakly downstream. This shows that the difference is stored in the trained checkpoints rather than introduced by the evaluation runtime. It does not establish why the training trajectories diverged. The defensible conclusion is that short-context convergence did not predict long-context behavior in these runs. Extrapolation studies should therefore report seeds, kernels, device classes, and extended-context audits; causal hardware claims require replicated, seed-paired interventions. Additional diagnostics are in Appendix~\ref{appendix:diagnostics}.

\subsection{What this audit cannot answer}
The finding opens several concrete questions that the present compute budget and page limit do not resolve:
\begin{enumerate}
\setlength{\itemsep}{0.15em}
\setlength{\parskip}{0pt}
\setlength{\parsep}{0pt}
\item Does the checkpoint divergence replicate under seed-paired device interventions, across software stacks, and at larger model scales?
\item Which training-level numerical difference---kernel selection, reduction or accumulation order, or optimizer-state evolution---is amplified into the divergent checkpoint, and can an early diagnostic predict it?
\item Does Polar's 0.3\% natural-text retrieval at 256K improve with longer or mixed-domain training, or reflect a synthetic-to-natural induction bias?
\item How much of Raven's 256K BABILong lead comes from recurrent state, sparse routing, or AdamW, given that the Muon schedule did not converge?
\end{enumerate}
These questions delimit the current evidence rather than extend its claims. Resolving them requires interventions that separately vary architecture, optimizer, data, and execution environment.

\section{Limitations and Conclusion}

This study uses one model scale, one 2K pretraining regime, and no seed replication sufficient for population-level architectural uncertainty. The Stage I factorial demonstrates robustness across nuisance-factor settings, not across random seeds. Teacher-forced retrieval, custom BABILong adaptation, and single-sample serving measurements answer different questions and should not be collapsed into one score. Polar has zero exact FinePDFs retrieval at 64K and 256K despite retaining partial target-token signal, its theoretical bounds require learned-margin assumptions not yet certified for every head, and Raven is not an optimizer-matched ablation. The documented pre-integration Polar design tests are not a substitute for network-level component ablations, and the present comparisons do not cover every contemporary length-scaling method.

Within those limits, ATMA establishes a reproducible operating point for bounded full-context attention plus recurrent memory. Polar materially extends target-token signal, preserves exact synthetic retrieval, and stabilizes long-document likelihood while remaining competitive on short-context tasks. Raven illustrates the BABILong and serving advantages of compressed recurrent state. The checkpoint audit further shows why length-extrapolation claims require evaluation far beyond the training window and careful reporting of run metadata. We release configurations, logs, checkpoint manifests, benchmark matrices, and numerical proof artifacts so that each reported comparison can be audited.

\subsection*{AI Use Statement}
Generative AI tools were used in this revision to edit and shorten the manuscript, improve document structure, and reaggregate existing archived results. They were not used to generate experimental datasets, train models, or run new evaluations in this revision. The authors manually compare all AI-assisted mathematical text with the original derivations and proof artifacts, verify empirical values against source logs, and take responsibility for every claim and artifact in the submission. This statement describes the present version as required by the ICLR 2027 policy.

\subsection*{Reproducibility Statement}
Model definitions and training settings are described in Sections 3--4. Appendix~\ref{appendix:protocol} specifies evaluation protocols and denominators, Appendices~\ref{appendix:full_ablation}--\ref{appendix:diagnostics} provide detailed tables and diagnostics. The public artifact at \url{https://github.com/kreasof-ai/atma} contains configurations, logs, checkpoint manifests, benchmark payloads, and machine-checked proof skeletons.

\clearpage
\bibliography{colm2026_conference,iclr2027_additions}

@inproceedings{Vaswani+2017,
 author = {Vaswani, Ashish and Shazeer, Noam and Parmar, Niki and Uszkoreit, Jakob and Jones, Llion and Gomez, Aidan N and Kaiser, \L ukasz and Polosukhin, Illia},
 booktitle = {Advances in Neural Information Processing Systems},
 pages = {5998--6008},
 publisher = {Curran Associates, Inc.},
 title = {Attention is All you Need},
 volume = {30},
 year = {2017}
}

@article{titans2025,
  title={Titans: Learning to Memorize at Test Time},
  author={Behrouz, Ali and Zhong, Peilin and Mirrokni, Vahab},
  journal={arXiv preprint arXiv:2501.00663},
  year={2025}
}

@article{lfm2_2025,
  title={LFM2 Technical Report},
  author={Hasani, Ramin and others},
  journal={arXiv preprint arXiv:2511.23404},
  year={2025}
}

@article{physics_lms_2025,
  title={Physics of Language Models: Part 4.1, Architecture Design and the Magic of Canon Layers},
  author={Allen-Zhu, Zeyuan},
  journal={arXiv preprint arXiv:2512.17351},
  year={2025}
}

@article{rope2021,
  title={Roformer: Enhanced transformer with rotary position embedding},
  author={Su, Jianlin and others},
  journal={arXiv preprint arXiv:2104.09864},
  year={2021}
}

@article{deltanet2024,
  title={Parallelizing Linear Transformers with the Delta Rule over Sequence Length},
  author={Yang, Songlin and Wang, Bailin and Zhang, Yu and Shen, Yikang and Kim, Yoon},
  journal={arXiv preprint arXiv:2406.06484},
  year={2024}
}

@article{yang2024gated,
  title={Gated Delta Networks: Improving Mamba2 with Delta Rule},
  author={Yang, Songlin and Kautz, Jan and Hatamizadeh, Ali},
  journal={arXiv preprint arXiv:2412.06464},
  year={2024}
}

@article{mamba2023,
  title={Mamba: Linear-time sequence modeling with selective state spaces},
  author={Gu, Albert and Dao, Tri},
  journal={arXiv preprint arXiv:2312.00752},
  year={2023}
}

@article{alibi2021,
  title={Train short, test long: Attention with linear biases enables input length extrapolation},
  author={Press, Ofir and Smith, Noah A and Lewis, Mike},
  journal={arXiv preprint arXiv:2108.12409},
  year={2021}
}

@article{nakanishi2025scalable,
  title={Scalable-Softmax Is Superior for Attention},
  author={Nakanishi, Ken M.},
  journal={arXiv preprint arXiv:2501.19399},
  year={2025}
}

@inproceedings{velickovic2025softmax,
  title={Softmax is not Enough (for Sharp Size Generalisation)},
  author={Velickovic, Petar and Perivolaropoulos, Christos and Barbero, Federico and Pascanu, Razvan},
  booktitle={International Conference on Machine Learning},
  year={2025}
}

@inproceedings{barbero2024glasses,
  title={Transformers Need Glasses! Information Over-Squashing in Language Tasks},
  author={Barbero, Federico and Banino, Andrea and Kapturowski, Steven and Kumaran, Dharshan and Araujo, Joao G. M. and Vitvitskyi, Alex and Pascanu, Razvan and Velickovic, Petar},
  booktitle={Advances in Neural Information Processing Systems},
  year={2024}
}

@inproceedings{vasylenko2026sparse,
  title={Long-Context Generalization with Sparse Attention},
  author={Vasylenko, Pavlo and Pitorro, Hugo and Martins, Andre F. T. and Treviso, Marcos},
  booktitle={International Conference on Learning Representations},
  year={2026}
}

@inproceedings{huang2026threshold,
  title={Threshold Differential Attention for Sink-Free, Ultra-Sparse, and Non-Dispersive Language Modeling},
  author={Huang, Xingyue and Ding, Xueying and Ju, Mingxuan and Liu, Yozen and Shah, Neil and Zhao, Tong},
  booktitle={Proceedings of the 64th Annual Meeting of the Association for Computational Linguistics},
  year={2026},
  url={https://arxiv.org/abs/2601.12145}
}

@article{lejepa2025,
  title={LeJEPA: Provable and Scalable Self-Supervised Learning Without the Heuristics},
  author={Balestriero, Randall and LeCun, Yann},
  journal={arXiv preprint arXiv:2511.08544},
  year={2025}
}

@article{weak_sigreg2026,
  title={Weak-SIGReg: Covariance Regularization for Stable Deep Learning},
  author={Akbar, Habibullah},
  journal={arXiv preprint arXiv:2603.05924},
  year={2026}
}

@inproceedings{babilong2024,
  title={BABILong: Testing the Limits of LLMs with Long Context Reasoning-in-a-Haystack},
  author={Kuratov, Yuri and Bulatov, Aydar and Anokhin, Petr and Rodkin, Ivan and Sorokin, Dmitry and Sorokin, Artyom and Burtsev, Mikhail},
  booktitle={Advances in Neural Information Processing Systems Datasets and Benchmarks Track},
  year={2024},
  url={https://arxiv.org/abs/2406.10149}
}

@article{raven2026,
  title={Raven: High-Recall Sequence Modeling with Sparse Memory Routing},
  author={Afzal, Arshia and Bick, Aviv and Xing, Eric P. and Cevher, Volkan and Gu, Albert},
  journal={arXiv preprint arXiv:2607.25357},
  year={2026},
  url={https://arxiv.org/abs/2607.25357}
}

@inproceedings{kwon2023vllm,
  title     = {Efficient Memory Management for Large Language Model Serving with PagedAttention},
  author    = {Kwon, Woosuk and Li, Zhuohan and Zhuang, Siyuan and Sheng, Ying and Zheng, Lianmin and Yu, Cody Hao and Gonzalez, Joseph E. and Zhang, Hao and Stoica, Ion},
  booktitle = {Proceedings of the 29th Symposium on Operating Systems Principles},
  pages     = {611--626},
  year      = {2023},
  doi       = {10.1145/3600006.3613165}
}
\bibliographystyle{iclr2027_conference}

\appendix
\section{Theoretical Framing}
\label{appendix:theory}

These are sufficient-condition results. They establish noise-odds and norm bounds under explicit assumptions. They do not establish distributional invariance or replace the empirical evidence.

\subsection{Extreme-value null floor}

\begin{assumption}[Sub-Gaussian noise scores]
For irrelevant keys $j\in\mathcal N_n$, $|\mathcal N_n|\leq n$, scores $X_j$ are mean-zero and $v^2$-sub-Gaussian: $\mathbb E\exp(tX_j)\leq\exp(v^2t^2/2)$ for all $t\in\mathbb R$.
\end{assumption}

Let the Polar null floor be $\nu_n=b+s\sqrt{\log(n+1)}$ with $b,s\geq0$.

\begin{proposition}[Noise exceedances]
Under Assumption 1,
\begin{equation}
\mathbb E\!\left[\sum_{j\in\mathcal N_n}{\bf1}\{X_j>\nu_n\}\right]
\leq n\exp\!\left(-\frac{\nu_n^2}{2v^2}\right).
\end{equation}
If $s\geq\sqrt2v$ this is $O(1)$; if the inequality is strict it decays polynomially.
\end{proposition}

\begin{proof}
Apply the sub-Gaussian tail bound to each key and sum. Since $b\geq0$, $\nu_n^2\geq s^2\log(n+1)$.
\end{proof}

Unlike TDA, Polar retains sub-threshold keys. With temperature $\theta_n=1+a\log n$, define the unnormalized noise-to-null odds $R_n=\sum_{j\in\mathcal N_n}\exp(\theta_n(X_j-\nu_n))$.

\begin{proposition}[Vanishing null odds]
Suppose $s>c>\sqrt2v$ and the null floor is $s\sqrt{\log(n+1)}$ up to a nonnegative offset. With probability tending to one,
\begin{equation}
R_n\leq n\exp\!\left(-(1+a\log n)(s-c)\sqrt{\log n}\right).
\end{equation}
For $a>0$, the bound converges to zero faster than any inverse polynomial.
\end{proposition}

\begin{proof}
A union bound gives $\Pr[\max_jX_j>c\sqrt{\log(n+1)}]\leq n(n+1)^{-c^2/(2v^2)}\to0$. On the complementary event, every noise score is at least $(s-c)\sqrt{\log n}$ below the floor. Bound each summand and sum over at most $n$ keys.
\end{proof}

\subsection{Bounded channels and recurrent state}

\begin{proposition}[Bounded Polar readout]
For $c_i=s_i/\max(\lVert s_i\rVert_2,\varepsilon)$, $\lVert c_i\rVert_2\leq1$. For $\mu_i=\tanh(\beta\log(1+m_{{\rm eff},i}))$, $\mu_i\in[0,1)$ when $\beta,m_{{\rm eff},i}\geq0$. A distribution uniform over $m$ real keys has participation ratio $m$.
\end{proposition}

\begin{proof}
The first two claims follow from the denominator and range of $\tanh$. Uniform weights satisfy $(\sum_{j=1}^m(1/m)^2)^{-1}=m$.
\end{proof}

\begin{assumption}[Contractive memory]
For every $t$, $\lVert q_t\rVert_2=\lVert k_t\rVert_2=1$, $\lVert v_t\rVert_2\leq V$, $0\leq\beta_t\leq1$, and $0\leq\gamma_t\leq\bar\gamma<1$.
\end{assumption}

\begin{proposition}[Length-bounded state]
For recurrence~\eqref{eq:memory},
\begin{equation}
\lVert M_t\rVert_F\leq\bar\gamma^t\lVert M_0\rVert_F+\frac{V}{1-\bar\gamma}.
\end{equation}
\end{proposition}

\begin{proof}
$I-\beta_tk_tk_t^\top$ has spectral norm at most one and $\lVert\beta_tv_tk_t^\top\rVert_F\leq V$. Thus $\lVert M_t\rVert_F\leq\bar\gamma\lVert M_{t-1}\rVert_F+V$; unrolling gives the result.
\end{proof}

The Lean artifact checks the finite-sum, exponential-monotonicity, bounded-channel composition, and invariant-set induction steps. It takes the sub-Gaussian tail, high-probability margin, and one-step contraction as premises. It is therefore a machine-checked proof skeleton, not end-to-end formal verification. The learned conditions $b\geq0$ and $s>c>\sqrt2v$ require an empirical head-wise audit before they can be asserted for a trained checkpoint.

\section{Regularization Modes}
\label{appendix:reg_modes}

The factorial grid varies a representation regularizer applied to the model's signature stream. For every non-baseline mode, the sweep uses the same weight $\alpha_{\rm sig}=0.01$ and optimizes
\begin{equation}
\mathcal L=(1-\alpha_{\rm sig})\mathcal L_{\rm LM}
+\alpha_{\rm sig}\mathcal L_{\rm reg}
+\lambda_{\rm dist}\mathcal L_{\rm dist}.
\end{equation}
The weak and strong modes follow the two SIGReg families introduced by Weak-SIGReg~\citep{weak_sigreg2026} and LeJEPA's strong SIGReg objective~\citep{lejepa2025}. The discrete and zipfian modes are local variants that combine those ideas with stronger geometric priors. Table~\ref{table:reg_modes} defines all five sweep settings.

\begin{table}[!ht]
\centering
\small
\begin{tabular}{p{0.17\linewidth}p{0.76\linewidth}}
\toprule
\textbf{Mode} & \textbf{Meaning in the sweep} \\
\midrule
\texttt{baseline} & No signature-stream regularization; equivalently, $\alpha_{\rm sig}=0$. \\
\texttt{weak} & Weak SIGReg: center a random sketch of the token representations and match its covariance to the identity, $\lVert\operatorname{Cov}(Sx)-I\rVert_F$. This encourages decorrelated, unit-variance features at low cost. \\
\texttt{strong} & Strong SIGReg: project representations onto random one-dimensional directions and match their empirical characteristic functions to the standard Gaussian characteristic function over a fixed frequency grid. \\
\texttt{discrete} & Local variant: normalize each representation to the sphere of radius $\sqrt d$, then apply weak covariance matching. This preserves whitening pressure while encouraging fixed-norm, discrete-code-like directions. \\
\texttt{zipfian} & Local variant: normalize directions for an orthogonality penalty, then match sorted representation magnitudes to a Zipf profile $r^{-1}$. This biases the signature stream toward heavy-tailed feature use rather than uniform energy allocation. \\
\bottomrule
\end{tabular}
\caption{Regularization modes in the 120-cell factorial sweep. Weak and strong are SIGReg-style covariance and distribution-matching objectives; discrete and zipfian are exploratory local variants.}
\label{table:reg_modes}
\end{table}
\FloatBarrier

\section{Polar Attention Implementation}
\label{appendix:polar_implementation}

\subsection{Design provenance and scope}
Before integration, the Polar choices were developed by falsifying simpler alternatives in a standalone synthetic sandbox. Additive soft counts accumulated small per-key leakage with sequence length; thresholds based on a fixed number of standard deviations retained a constant fraction of noise keys; a fixed null logit was eventually overtaken by the maximum noise score; and an unbounded $\log m$ magnitude exposed the output projection to values beyond its training range. These failures motivated participation ratio, the $\sqrt{\log n}$ null floor, and the bounded $\tanh(\beta\log(1+m))$ map. The standalone prototype was subsequently removed and is not part of the released artifact. We therefore use these tests as design provenance, not as model-scale component ablations or independently reproducible benchmark evidence. The integrated equations, materialized oracle, streaming implementation, and parity tests described below are retained.

\subsection{Numerical oracle}
The materialized PyTorch path forms the masked real-key logits and virtual null logit in FP32, applies a single softmax, and computes the direction and magnitude channels from the resulting weights. It is used as a numerical oracle for forward and backward tests, not as the production long-context path, because materializing the $T\times T$ score matrix is memory-prohibitive at the evaluated lengths.

\subsection{Streaming sufficient statistics}
The Triton kernel follows the online-max strategy used by FlashAttention-style kernels. For one query row, it streams $K,V$ blocks and retains only four FP32 statistics: the running maximum $M$, real-key exponential sum $L$, squared exponential sum $Q_2$, and unnormalized value sum $S$. If a new block changes the row maximum, $\alpha=\exp(M_{\rm old}-M_{\rm new})$ rescales each statistic as shown in Table~\ref{table:polar_streaming_state}.

\begin{table}[!ht]
\centering
\small
\begin{tabular}{lll}
\toprule
\textbf{State} & \textbf{Online update} & \textbf{Use after streaming} \\
\midrule
$M$ & $M\leftarrow\max(M,\max a)$ & Stable common logit origin \\
$L$ & $L\leftarrow\alpha L+\sum_jp_j$ & Real-key softmax mass \\
$Q_2$ & $Q_2\leftarrow\alpha^2Q_2+\sum_jp_j^2$ & $n_{\rm eff}=L^2/Q_2$ \\
$S$ & $S\leftarrow\alpha S+\sum_jp_jv_j$ & Direction numerator \\
\bottomrule
\end{tabular}
\caption{Streaming state for one Polar query row, where $a$ is the current logit block and $p_j=\exp(a_j-M_{\rm new})$. The squared accumulator must rescale by $\alpha^2$.}
\label{table:polar_streaming_state}
\end{table}

After the real keys are consumed, the virtual null is folded in as one additional softmax entry. Let $p_0=\exp(t_i\nu_i-M)$ and $Z=L+p_0$. The kernel obtains
\begin{equation}
c_i=\operatorname{normalize}(S+p_0v_{\rm null}),\qquad
n_{{\rm eff},i}=\frac{L^2}{Q_2},\qquad
m_{{\rm eff},i}=n_{{\rm eff},i}\frac{L}{Z}.
\end{equation}
The common factor $1/Z$ cancels under radial normalization of the direction numerator. The production kernel evaluates the magnitude with $2\sigma(2x)-1=\tanh x$ and saves $(M,L,Q_2,S)$ for the backward pass. Listing~\ref{lst:polar_triton} gives the algorithmic core; pointer arithmetic, launch metadata, and backward code are omitted.

\begin{lstlisting}[language=Python,caption={Schematic core of the block-streamed Triton forward pass.},label={lst:polar_triton}]
@triton.jit
def polar_fwd_core(q, K, V, n_i, temp, nullv, beta,
                   BLOCK_N: tl.constexpr, DK: tl.constexpr):
    M  = tl.full([BLOCK_M], -1e38, tl.float32)
    L  = tl.zeros([BLOCK_M], tl.float32)
    Q2 = tl.zeros([BLOCK_M], tl.float32)
    S  = tl.zeros([BLOCK_M, DK], tl.float32)

    for start in range(0, n_i, BLOCK_N):
        k, v, valid = load_kv_block(K, V, start, n_i)
        a = tl.where(valid, temp[:, None] * dot(q, k), -1e38)
        M_new = tl.maximum(M, tl.max(a, axis=1))
        alpha = tl.exp(M - M_new)
        p = tl.where(valid, tl.exp(a - M_new[:, None]), 0.0)
        L  = alpha * L + tl.sum(p, axis=1)
        Q2 = alpha * alpha * Q2 + tl.sum(p * p, axis=1)
        S  = alpha[:, None] * S + tl.dot(p, v)
        M = M_new

    # Fold in the learned virtual null as one extra entry.
    M_new = tl.maximum(M, temp * nullv)
    alpha = tl.exp(M - M_new)
    L, Q2, S = alpha * L, alpha * alpha * Q2, alpha[:, None] * S
    p_null = tl.exp(temp * nullv - M_new)
    Z = L + p_null

    direction = normalize(S + p_null[:, None] * v_null)
    n_eff = L * L / tl.maximum(Q2, eps)
    m_eff = n_eff * L / tl.maximum(Z, eps)
    magnitude = 2 * tl.sigmoid(2 * beta * tl.log(1 + m_eff)) - 1
    save_for_backward(M_new, L, Q2, S)
\end{lstlisting}

This removes the quadratic \emph{score-matrix storage}. Full-context prefill still performs quadratic query-key arithmetic. The kernel is cross-checked against the materialized PyTorch oracle in FP32 accumulation before lower-precision outputs are accepted.

\section{Evaluation Protocols}
\label{appendix:protocol}

Stage I contains 120 approximately 370--378M-parameter runs trained for 1,900 steps of 524,288 tokens at sequence length 2,048 on FineWeb-Edu. Stage II contains three matched attention models trained for 18,722 such steps (9.816B tokens). The five promoted checkpoints and ten open-weight controls are evaluated with checkpoint-exact forward paths. Stage I cells vary factorial settings and are not independent seed replicates.

\subsection{Raven optimizer control}
The Raven rows use AdamW optimizer because the attempted matched-schedule control did not produce a usable baseline. In the 1B-token Raven Native pilot with the Atma Muon schedule, validation loss decreased from 10.83 to 3.76 nats by step 625, then rose to 6.88 at step 1,125 and ended at 4.62; teacher-forced retrieval was 0\% at every tested length from 2K through 64K. The reported Raven Native and Atma-Raven-Titans checkpoints therefore use AdamW and remain a separate model-family/optimizer group. This is evidence about the tested Muon schedule, not a general claim that Raven cannot converge under any Muon tuning.

Retrieval crosses passkey and NIAH tasks, synthetic and FinePDFs haystacks, eight lengths from 2K through 256K, three depths, and 50 paired trials per cell. The scorer appends the complete five-token target and evaluates each position conditioned on the preceding ground-truth target tokens. Token accuracy is the fraction of correct per-position argmax predictions; exact accuracy requires all five positions to be correct. Both are teacher-forced diagnostics, not autoregressive generation. An older free-generation protocol was retired because all base checkpoints scored uniformly zero. BABILong uses answer-only fine-tuning on tasks qa1--qa10 at 0K/1K/2K and greedy exact match on ten held-out rows per task and length. Base-task controls use LAMBADA, HellaSwag, PIQA, WinoGrande, ARC-Easy, ARC-Challenge, OpenBookQA, and BoolQ. Fixed-target BPB uses FinePDFs, PG-19, and Proof-Pile. All reported processes completed without failed, OOM, unsupported, missing, or null cells.

For proportions we report simple Wilson intervals where they aid interpretation. These intervals treat trials as binomial observations and do not exploit the paired design or account for correlations induced by shared templates. The archived retrieval payloads retain aggregate cell counts but the BABILong payloads do not retain per-model correctness vectors, so paired bootstrap or McNemar tests cannot be reconstructed without rerunning or recovering the original predictions.

\begin{landscape}
\section{Complete Factorial Sweep}
\label{appendix:full_ablation}

Table~\ref{table:ablation_appendix} retains all 120 factorial cells; the main paper reports only the recipe transitions needed to justify promotion.
\begingroup
\normalsize
\setlength{\tabcolsep}{5pt}
\begin{longtable}{llcccrrrrrrrr}
\caption{Complete Results of the 120-Run Ablation Sweep. All models are 370M parameters, trained on FineWeb-Edu for 1B tokens ($seq\_len=2048$). Evaluated from 2K to 64K context length. Clean NLL is measured on coherent FinePDFs documents; junk NLL is measured on the concatenated validation stream. NLL is in nats/token (lower is better), Needle is greedy per-digit accuracy in \% (higher is better), and MFU is training model flops utilization in \%. \label{table:ablation_appendix}}\\
\toprule
Type & Reg & Dist & Mem & Win & Val & Clean2K & Clean64K & Junk2K & Junk64K & Ndl2K & Ndl64K & MFU \\
\midrule
\endfirsthead
\caption[]{Complete Results of the 120-Run Ablation Sweep (continued).}\\
\toprule
Type & Reg & Dist & Mem & Win & Val & Clean2K & Clean64K & Junk2K & Junk64K & Ndl2K & Ndl64K & MFU \\
\midrule
\endhead
\midrule
\multicolumn{13}{r}{{Continued on next page}} \\
\bottomrule
\endfoot
\bottomrule
\endlastfoot
POLAR & baseline & Off & Off & Off & 3.323 & 2.83 & 3.60 & 3.27 & 5.68 & 96.3 & 0.0 & 40.1 \\
POLAR & baseline & Off & Off & On & 3.343 & 2.83 & 2.66 & 3.30 & 4.74 & 67.5 & 0.0 & 39.6 \\
POLAR & baseline & Off & On & Off & 3.169 & 2.70 & 1.96 & 3.11 & 3.13 & 91.3 & 92.5 & 35.6 \\
POLAR & baseline & Off & On & On & 3.174 & 2.71 & 2.01 & 3.12 & 3.13 & 51.3 & 30.0 & 35.6 \\
POLAR & baseline & On & Off & Off & 3.332 & 2.81 & 5.56 & 3.28 & 6.54 & 85.0 & 0.0 & 33.3 \\
POLAR & baseline & On & Off & On & 3.361 & 2.85 & 2.89 & 3.35 & 4.97 & 91.3 & 2.5 & 34.5 \\
POLAR & baseline & On & On & Off & 3.178 & 2.72 & 1.98 & 3.12 & 3.14 & 73.8 & 58.8 & 31.8 \\
POLAR & baseline & On & On & On & 3.182 & 2.73 & 1.97 & 3.12 & 3.14 & 56.3 & 21.3 & 31.2 \\
POLAR & weak & Off & Off & Off & 3.333 & 2.82 & 2.86 & 3.28 & 5.35 & 88.8 & 0.0 & 39.0 \\
POLAR & weak & Off & Off & On & 3.348 & 2.90 & 3.76 & 3.38 & 5.24 & 93.8 & 0.0 & 37.9 \\
POLAR & weak & Off & On & Off & 3.179 & 2.74 & 1.95 & 3.12 & 3.14 & 70.0 & 51.3 & 36.2 \\
POLAR & weak & Off & On & On & 3.187 & 2.74 & 1.99 & 3.13 & 3.15 & 41.3 & 41.3 & 36.2 \\
POLAR & weak & On & Off & Off & 3.335 & 2.82 & 3.47 & 3.28 & 5.56 & 91.3 & 0.0 & 32.8 \\
POLAR & weak & On & Off & On & 3.353 & 2.82 & 2.58 & 3.32 & 4.81 & 96.3 & 1.3 & 32.8 \\
POLAR & weak & On & On & Off & 3.178 & 2.74 & 1.95 & 3.12 & 3.14 & 33.8 & 32.5 & 30.7 \\
POLAR & weak & On & On & On & 3.183 & 2.74 & 1.94 & 3.13 & 3.15 & 81.3 & 78.8 & 30.7 \\
POLAR & strong & Off & Off & Off & 3.338 & 2.85 & 3.79 & 3.28 & 5.83 & 98.8 & 1.3 & 37.3 \\
POLAR & strong & Off & Off & On & 3.360 & 2.85 & 2.99 & 3.32 & 5.01 & 91.3 & 0.0 & 37.9 \\
POLAR & strong & Off & On & Off & 3.172 & 2.72 & 1.96 & 3.11 & 3.13 & 80.0 & 52.5 & 35.8 \\
POLAR & strong & Off & On & On & 3.181 & 2.72 & 1.94 & 3.12 & 3.15 & 76.3 & 68.8 & 34.6 \\
POLAR & strong & On & Off & Off & 3.337 & 2.87 & 3.22 & 3.28 & 5.52 & 83.8 & 1.3 & 32.8 \\
POLAR & strong & On & Off & On & 3.347 & 2.82 & 2.64 & 3.30 & 4.88 & 97.5 & 16.3 & 32.2 \\
POLAR & strong & On & On & Off & 3.179 & 2.73 & 1.94 & 3.12 & 3.14 & 86.3 & 83.8 & 30.9 \\
POLAR & strong & On & On & On & 3.182 & 2.72 & 1.94 & 3.12 & 3.14 & 25.0 & 23.8 & 30.9 \\
POLAR & discrete & Off & Off & Off & 3.338 & 2.82 & 2.80 & 3.29 & 5.12 & 93.8 & 2.5 & 37.9 \\
POLAR & discrete & Off & Off & On & 3.349 & 2.87 & 2.90 & 3.33 & 5.25 & 93.8 & 6.3 & 37.9 \\
POLAR & discrete & Off & On & Off & 3.184 & 2.75 & 2.03 & 3.12 & 3.14 & 76.3 & 73.8 & 36.4 \\
POLAR & discrete & Off & On & On & 3.178 & 2.71 & 1.92 & 3.12 & 3.15 & 31.3 & 18.8 & 35.8 \\
POLAR & discrete & On & Off & Off & 3.333 & 2.80 & 4.13 & 3.28 & 6.12 & 97.5 & 0.0 & 32.6 \\
POLAR & discrete & On & Off & On & 3.351 & 2.87 & 2.72 & 3.33 & 4.89 & 85.0 & 2.5 & 33.3 \\
POLAR & discrete & On & On & Off & 3.189 & 2.74 & 1.96 & 3.13 & 3.15 & 87.5 & 53.8 & 31.8 \\
POLAR & discrete & On & On & On & 3.192 & 2.74 & 2.07 & 3.13 & 3.15 & 21.3 & 26.3 & 31.7 \\
POLAR & zipfian & Off & Off & Off & 3.338 & 2.85 & 4.83 & 3.28 & 6.55 & 96.3 & 11.3 & 37.9 \\
POLAR & zipfian & Off & Off & On & 3.350 & 2.83 & 2.58 & 3.32 & 4.63 & 88.8 & 3.8 & 39.4 \\
POLAR & zipfian & Off & On & Off & 3.171 & 2.70 & 1.92 & 3.11 & 3.13 & 52.5 & 63.8 & 36.2 \\
POLAR & zipfian & Off & On & On & 3.180 & 2.75 & 1.95 & 3.12 & 3.15 & 46.3 & 40.0 & 35.2 \\
POLAR & zipfian & On & Off & Off & 3.340 & 2.83 & 3.77 & 3.28 & 5.64 & 98.8 & 1.3 & 33.8 \\
POLAR & zipfian & On & Off & On & 3.352 & 2.85 & 3.13 & 3.36 & 5.23 & 90.0 & 0.0 & 33.5 \\
POLAR & zipfian & On & On & Off & 3.177 & 2.72 & 1.94 & 3.12 & 3.14 & 91.3 & 67.5 & 31.8 \\
POLAR & zipfian & On & On & On & 3.192 & 2.74 & 2.00 & 3.13 & 3.15 & 40.0 & 26.3 & 31.8 \\
NOPE & baseline & Off & Off & Off & 3.224 & 2.76 & 3.71 & 3.16 & 6.68 & 97.5 & 1.3 & 41.5 \\
NOPE & baseline & Off & Off & On & 3.219 & 2.75 & 2.76 & 3.17 & 5.49 & 92.5 & 10.0 & 36.5 \\
NOPE & baseline & Off & On & Off & 3.140 & 2.66 & 2.34 & 3.09 & 3.23 & 97.5 & 16.3 & 36.8 \\
NOPE & baseline & Off & On & On & 3.150 & 2.70 & 2.14 & 3.09 & 3.24 & 81.3 & 0.0 & 34.2 \\
NOPE & baseline & On & Off & Off & 3.209 & 2.74 & 2.99 & 3.15 & 6.13 & 90.0 & 18.8 & 32.5 \\
NOPE & baseline & On & Off & On & 3.212 & 2.77 & 2.90 & 3.17 & 5.65 & 81.3 & 0.0 & 30.3 \\
NOPE & baseline & On & On & Off & 3.149 & 2.67 & 2.35 & 3.10 & 3.28 & 80.0 & 16.3 & 30.0 \\
NOPE & baseline & On & On & On & 3.147 & 2.68 & 2.09 & 3.09 & 3.23 & 88.8 & 21.3 & 28.2 \\
NOPE & weak & Off & Off & Off & 3.214 & 2.77 & 3.10 & 3.16 & 6.34 & 95.0 & 0.0 & 40.9 \\
NOPE & weak & Off & Off & On & 3.222 & 2.75 & 3.05 & 3.19 & 5.81 & 81.3 & 0.0 & 36.1 \\
NOPE & weak & Off & On & Off & 3.147 & 2.68 & 2.07 & 3.09 & 3.19 & 88.8 & 21.3 & 37.5 \\
NOPE & weak & Off & On & On & 3.142 & 2.68 & 2.12 & 3.08 & 3.23 & 82.5 & 0.0 & 33.7 \\
NOPE & weak & On & Off & Off & 3.215 & 2.75 & 4.87 & 3.16 & 12.48 & 95.0 & 5.0 & 32.0 \\
NOPE & weak & On & Off & On & 3.221 & 2.73 & 2.63 & 3.18 & 4.84 & 81.3 & 0.0 & 29.3 \\
NOPE & weak & On & On & Off & 3.147 & 2.67 & 2.48 & 3.09 & 3.23 & 97.5 & 1.3 & 30.8 \\
NOPE & weak & On & On & On & 3.148 & 2.68 & 2.14 & 3.09 & 3.24 & 93.8 & 1.3 & 28.4 \\
NOPE & strong & Off & Off & Off & 3.231 & 2.74 & 4.88 & 3.18 & 11.85 & 90.0 & 2.5 & 38.4 \\
NOPE & strong & Off & Off & On & 3.214 & 2.72 & 2.68 & 3.17 & 5.28 & 82.5 & 0.0 & 36.4 \\
NOPE & strong & Off & On & Off & 3.144 & 2.68 & 2.37 & 3.09 & 3.24 & 96.3 & 6.3 & 35.7 \\
NOPE & strong & Off & On & On & 3.150 & 2.70 & 2.10 & 3.10 & 3.25 & 93.8 & 7.5 & 33.8 \\
NOPE & strong & On & Off & Off & 3.207 & 2.72 & 3.13 & 3.15 & 6.20 & 73.8 & 0.0 & 31.4 \\
NOPE & strong & On & Off & On & 3.214 & 2.74 & 2.49 & 3.17 & 4.36 & 86.3 & 8.8 & 28.9 \\
NOPE & strong & On & On & Off & 3.146 & 2.67 & 2.24 & 3.09 & 3.23 & 83.8 & 17.5 & 29.1 \\
NOPE & strong & On & On & On & 3.147 & 2.68 & 2.16 & 3.09 & 3.28 & 96.3 & 0.0 & 27.5 \\
NOPE & discrete & Off & Off & Off & 3.202 & 2.75 & 3.33 & 3.15 & 6.23 & 85.0 & 1.3 & 38.8 \\
NOPE & discrete & Off & Off & On & 3.234 & 2.79 & 3.11 & 3.19 & 5.47 & 70.0 & 18.8 & 37.0 \\
NOPE & discrete & Off & On & Off & 3.145 & 2.68 & 2.14 & 3.09 & 3.27 & 98.8 & 16.3 & 36.0 \\
NOPE & discrete & Off & On & On & 3.145 & 2.68 & 2.10 & 3.09 & 3.25 & 95.0 & 3.8 & 34.3 \\
NOPE & discrete & On & Off & Off & 3.212 & 2.88 & 3.04 & 3.16 & 6.88 & 42.5 & 0.0 & 31.2 \\
NOPE & discrete & On & Off & On & 3.211 & 2.78 & 2.73 & 3.16 & 5.28 & 82.5 & 1.3 & 29.1 \\
NOPE & discrete & On & On & Off & 3.145 & 2.68 & 2.27 & 3.08 & 3.26 & 91.3 & 11.3 & 30.3 \\
NOPE & discrete & On & On & On & 3.149 & 2.68 & 2.17 & 3.08 & 3.32 & 83.8 & 0.0 & 28.6 \\
NOPE & zipfian & Off & Off & Off & 3.209 & 2.71 & 3.18 & 3.15 & 6.40 & 90.0 & 0.0 & 40.8 \\
NOPE & zipfian & Off & Off & On & 3.219 & 2.75 & 2.71 & 3.17 & 5.44 & 77.5 & 1.3 & 36.1 \\
NOPE & zipfian & Off & On & Off & 3.141 & 2.67 & 2.55 & 3.08 & 3.38 & 96.3 & 0.0 & 38.0 \\
NOPE & zipfian & Off & On & On & 3.150 & 2.69 & 2.11 & 3.09 & 3.25 & 90.0 & 3.8 & 33.7 \\
NOPE & zipfian & On & Off & Off & 3.207 & 2.77 & 3.32 & 3.15 & 7.09 & 82.5 & 0.0 & 31.3 \\
NOPE & zipfian & On & Off & On & 3.227 & 2.79 & 2.91 & 3.18 & 5.94 & 93.8 & 11.3 & 30.2 \\
NOPE & zipfian & On & On & Off & 3.146 & 2.68 & 2.39 & 3.10 & 3.25 & 96.3 & 21.3 & 29.6 \\
NOPE & zipfian & On & On & On & 3.148 & 2.68 & 2.17 & 3.09 & 3.28 & 90.0 & 6.3 & 28.7 \\
ROPE & baseline & Off & Off & Off & 3.159 & 2.85 & 3.36 & 3.09 & 5.21 & 42.5 & 0.0 & 40.4 \\
ROPE & baseline & Off & Off & On & 3.163 & 2.86 & 3.79 & 3.17 & 5.86 & 15.0 & 0.0 & 37.2 \\
ROPE & baseline & Off & On & Off & 3.145 & 2.81 & 2.86 & 3.09 & 3.57 & 73.8 & 0.0 & 37.1 \\
ROPE & baseline & Off & On & On & 3.151 & 2.80 & 2.95 & 3.12 & 3.72 & 13.8 & 0.0 & 34.4 \\
ROPE & baseline & On & Off & Off & 3.163 & 3.00 & 3.47 & 3.10 & 5.46 & 38.8 & 0.0 & 33.2 \\
ROPE & baseline & On & Off & On & 3.167 & 2.88 & 4.06 & 3.16 & 5.92 & 0.0 & 0.0 & 30.2 \\
ROPE & baseline & On & On & Off & 3.153 & 2.73 & 2.45 & 3.09 & 3.57 & 75.0 & 0.0 & 31.4 \\
ROPE & baseline & On & On & On & 3.154 & 2.83 & 2.86 & 3.12 & 3.72 & 32.5 & 0.0 & 29.4 \\
ROPE & weak & Off & Off & Off & 3.158 & 2.88 & 3.46 & 3.10 & 5.36 & 46.3 & 0.0 & 41.4 \\
ROPE & weak & Off & Off & On & 3.163 & 2.84 & 3.96 & 3.18 & 5.93 & 12.5 & 0.0 & 37.8 \\
ROPE & weak & Off & On & Off & 3.149 & 2.71 & 2.46 & 3.09 & 3.56 & 37.5 & 0.0 & 36.5 \\
ROPE & weak & Off & On & On & 3.156 & 2.85 & 2.77 & 3.12 & 3.60 & 23.8 & 0.0 & 34.9 \\
ROPE & weak & On & Off & Off & 3.155 & 2.83 & 3.49 & 3.10 & 5.27 & 38.8 & 0.0 & 32.3 \\
ROPE & weak & On & Off & On & 3.157 & 3.05 & 4.12 & 3.18 & 5.86 & 8.8 & 0.0 & 30.5 \\
ROPE & weak & On & On & Off & 3.149 & 2.77 & 2.76 & 3.09 & 3.56 & 57.5 & 0.0 & 31.1 \\
ROPE & weak & On & On & On & 3.149 & 2.77 & 3.13 & 3.12 & 3.80 & 17.5 & 0.0 & 28.0 \\
ROPE & strong & Off & Off & Off & 3.165 & 2.78 & 3.34 & 3.10 & 5.27 & 30.0 & 0.0 & 39.0 \\
ROPE & strong & Off & Off & On & 3.169 & 2.90 & 3.68 & 3.19 & 5.80 & 31.3 & 0.0 & 36.3 \\
ROPE & strong & Off & On & Off & 3.158 & 2.85 & 2.63 & 3.10 & 3.50 & 51.3 & 0.0 & 35.9 \\
ROPE & strong & Off & On & On & 3.162 & 2.86 & 3.08 & 3.14 & 3.76 & 10.0 & 0.0 & 33.2 \\
ROPE & strong & On & Off & Off & 3.161 & 2.85 & 3.58 & 3.10 & 5.28 & 47.5 & 0.0 & 31.1 \\
ROPE & strong & On & Off & On & 3.165 & 2.85 & 4.03 & 3.19 & 5.90 & 2.5 & 0.0 & 29.2 \\
ROPE & strong & On & On & Off & 3.155 & 2.80 & 2.70 & 3.10 & 3.62 & 61.3 & 0.0 & 29.3 \\
ROPE & strong & On & On & On & 3.162 & 2.85 & 3.08 & 3.13 & 3.74 & 2.5 & 0.0 & 28.4 \\
ROPE & discrete & Off & Off & Off & 3.165 & 2.91 & 3.43 & 3.11 & 5.42 & 22.5 & 0.0 & 39.7 \\
ROPE & discrete & Off & Off & On & 3.158 & 2.84 & 4.05 & 3.15 & 5.90 & 3.8 & 0.0 & 36.9 \\
ROPE & discrete & Off & On & Off & 3.159 & 2.83 & 2.67 & 3.11 & 3.56 & 47.5 & 0.0 & 36.3 \\
ROPE & discrete & Off & On & On & 3.161 & 2.79 & 2.71 & 3.13 & 3.74 & 41.3 & 0.0 & 33.7 \\
ROPE & discrete & On & Off & Off & 3.171 & 2.79 & 3.45 & 3.12 & 5.36 & 51.3 & 0.0 & 31.5 \\
ROPE & discrete & On & Off & On & 3.164 & 2.80 & 3.74 & 3.16 & 5.86 & 21.3 & 0.0 & 30.0 \\
ROPE & discrete & On & On & Off & 3.162 & 2.74 & 2.47 & 3.11 & 3.53 & 50.0 & 0.0 & 30.9 \\
ROPE & discrete & On & On & On & 3.154 & 2.79 & 2.68 & 3.11 & 3.67 & 18.8 & 0.0 & 28.9 \\
ROPE & zipfian & Off & Off & Off & 3.163 & 2.82 & 3.25 & 3.10 & 5.22 & 62.5 & 0.0 & 39.8 \\
ROPE & zipfian & Off & Off & On & 3.161 & 2.94 & 3.94 & 3.18 & 5.81 & 16.3 & 0.0 & 37.1 \\
ROPE & zipfian & Off & On & Off & 3.149 & 2.78 & 2.70 & 3.09 & 3.63 & 46.3 & 0.0 & 36.7 \\
ROPE & zipfian & Off & On & On & 3.151 & 2.80 & 2.63 & 3.11 & 3.68 & 2.5 & 0.0 & 33.9 \\
ROPE & zipfian & On & Off & Off & 3.161 & 2.78 & 3.38 & 3.10 & 5.32 & 47.5 & 0.0 & 32.6 \\
ROPE & zipfian & On & Off & On & 3.160 & 3.01 & 4.11 & 3.17 & 5.81 & 17.5 & 0.0 & 30.2 \\
ROPE & zipfian & On & On & Off & 3.149 & 2.73 & 2.39 & 3.09 & 3.55 & 62.5 & 0.0 & 29.8 \\
ROPE & zipfian & On & On & On & 3.152 & 2.79 & 2.99 & 3.12 & 3.78 & 26.3 & 0.0 & 28.5 \\
\end{longtable}

\endgroup
\end{landscape}

\section{Retrieval Breakdown}
\label{appendix:retrieval}

\begin{figure}[h]
\centering
\includegraphics[width=\linewidth]{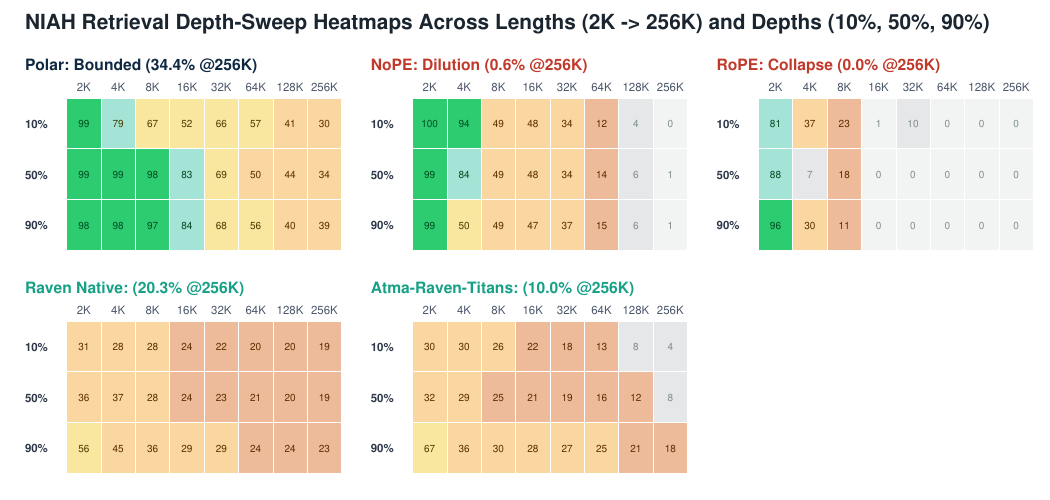}
\caption{Retrieval depth sweeps for the five promoted models. Each cell is teacher-forced token accuracy.}
\label{fig:retrieval_depth_sweeps}
\end{figure}

\begin{table}[h]
\centering
\small
\resizebox{\linewidth}{!}{%
\begin{tabular}{llrrrrrrrr}
\toprule
\textbf{Model} & \textbf{Haystack} & \textbf{2K} & \textbf{4K} & \textbf{8K} & \textbf{16K} & \textbf{32K} & \textbf{64K} & \textbf{128K} & \textbf{256K} \\
\midrule
Polar & Synthetic & 99.6 & 98.4 & 98.5 & 97.0 & 94.9 & 92.3 & 80.3 & 68.4 \\
& FinePDFs & 98.6 & 86.3 & 77.9 & 49.1 & 40.7 & 16.3 & 3.2 & 0.3 \\
NoPE & Synthetic & 99.5 & 98.7 & 98.9 & 94.9 & 70.3 & 26.5 & 10.2 & 1.2 \\
& FinePDFs & 99.3 & 52.9 & 0.0 & 0.0 & 0.0 & 0.0 & 0.0 & 0.0 \\
RoPE & Synthetic & 81.7 & 44.9 & 33.1 & 0.8 & 6.8 & 0.1 & 0.0 & 0.0 \\
& FinePDFs & 94.8 & 4.7 & 0.2 & 0.0 & 0.0 & 0.0 & 0.0 & 0.0 \\
Atma-Raven-Titans & Synthetic & 64.8 & 60.7 & 51.1 & 46.9 & 41.7 & 35.8 & 25.9 & 19.3 \\
& FinePDFs & 23.2 & 2.6 & 0.7 & 0.2 & 1.4 & 0.5 & 1.6 & 0.7 \\
Raven Native & Synthetic & 55.9 & 57.7 & 51.4 & 48.3 & 44.3 & 40.9 & 38.9 & 35.1 \\
& FinePDFs & 22.1 & 15.4 & 8.2 & 3.3 & 5.3 & 2.9 & 3.7 & 5.4 \\
\bottomrule
\end{tabular}}
\caption{Zero-shot retrieval accuracy (\%) by haystack type and context length, averaged over needle depths.}
\label{table:retrieval_breakdown}
\end{table}

\begin{table}[h]
\centering
\small
\setlength{\tabcolsep}{4pt}
\resizebox{\linewidth}{!}{%
\begin{tabular}{llrrrrrr}
\toprule
\textbf{Model} & \textbf{Haystack} & \multicolumn{3}{c}{\textbf{Token accuracy (\%)}} & \multicolumn{3}{c}{\textbf{Exact five-token accuracy (\%)}} \\
\cmidrule(lr){3-5}\cmidrule(lr){6-8}
& & \textbf{2K} & \textbf{64K} & \textbf{256K} & \textbf{2K} & \textbf{64K} & \textbf{256K} \\
\midrule
Polar & Synthetic & 99.6 & 92.3 & 68.4 & 96.3 & 65.7 & 18.0 \\
& FinePDFs & 98.6 & 16.3 & 0.3 & 93.0 & 0.0 & 0.0 \\
NoPE & Synthetic & 99.5 & 26.5 & 1.2 & 95.3 & 4.7 & 0.0 \\
& FinePDFs & 99.3 & 0.0 & 0.0 & 96.7 & 0.0 & 0.0 \\
RoPE & Synthetic & 81.7 & 0.1 & 0.0 & 42.3 & 0.0 & 0.0 \\
& FinePDFs & 94.8 & 0.0 & 0.0 & 77.3 & 0.0 & 0.0 \\
Atma-Raven-Titans & Synthetic & 64.8 & 35.8 & 19.3 & 8.0 & 0.3 & 0.0 \\
& FinePDFs & 23.2 & 0.5 & 0.7 & 4.7 & 0.0 & 0.0 \\
Raven Native & Synthetic & 55.9 & 40.9 & 35.1 & 2.0 & 0.0 & 0.0 \\
& FinePDFs & 22.1 & 2.9 & 5.4 & 0.0 & 0.0 & 0.0 \\
\bottomrule
\end{tabular}}
\caption{Teacher-forced retrieval under the two archived metrics. Token accuracy can expose partial target signal, whereas exact accuracy requires the entire five-token value to be predicted correctly under teacher forcing. Each haystack--length entry averages 300 sequences (two tasks, three depths, 50 trials).}
\label{table:retrieval_exact_breakdown}
\end{table}

Figure~\ref{fig:retrieval_depth_sweeps} exposes depth-wise variation. Tables~\ref{table:retrieval_breakdown} and~\ref{table:retrieval_exact_breakdown} separate partial target signal from exact five-token success. For Polar, the 64K exact rate is 65.7\% on synthetic contexts (95\% Wilson interval 60.1--70.8\%) and 0/300 on FinePDFs (upper Wilson endpoint 1.3\%). At 256K the corresponding exact rates are 18.0\% (14.1--22.7\%) and 0/300. These results support exact synthetic extrapolation and teacher-forced real-text signal, but not exact real-text retrieval.

\subsection{Open-weight context}
Table~\ref{table:open_weight_controls} reports the ten archived open-weight controls. They differ in training data, tokenizer, architecture, parameter count, and native context length; several were trained at lengths beyond 2K. They are therefore practical context, not matched length-extrapolation baselines and not evidence of a Polar advantage over the corresponding model families.

\begin{table}[h]
\centering
\scriptsize
\setlength{\tabcolsep}{3.4pt}
\begin{tabular}{lrrrr}
\toprule
\textbf{Open-weight model} & \textbf{NLL 2K} & \textbf{NLL 64K} & \textbf{Needle 2K} & \textbf{Needle 64K} \\
\midrule
SmolLM2-360M & 1.99 & 6.04 & 100.0 & 8.1 \\
LFM2.5-230M & 3.01 & 6.69 & 38.1 & 19.4 \\
LFM2.5-350M & 3.18 & 6.63 & 45.6 & 46.3 \\
Qwen2.5-0.5B & 2.03 & 1.53 & 100.0 & 38.1 \\
Qwen3-0.6B-Base & 1.95 & 1.40 & 100.0 & 95.0 \\
Qwen3.5-0.8B-Base & 1.94 & 1.33 & 100.0 & 100.0 \\
Gemma-3-270M & 2.48 & 1.90 & 96.2 & 24.4 \\
Granite-4.0-350M & 2.52 & 3.02 & 40.6 & 8.8 \\
Granite-4.0-H-350M & 2.22 & 2.06 & 75.0 & 54.4 \\
Falcon-H1-0.5B & 1.99 & 2.54 & 92.5 & 19.4 \\
\bottomrule
\end{tabular}
\caption{Open-weight controls evaluated through 64K. NLL is fixed-target FinePDFs nats/token; Needle is teacher-forced target-token accuracy over 16 trials. These checkpoints are not matched for data, tokenizer, native context, or training budget.}
\label{table:open_weight_controls}
\end{table}

\section{Short-Context Quality and Systems}
\label{appendix:quality}

\begin{table}[!ht]
\centering
\scriptsize
\setlength{\tabcolsep}{2.4pt}
\resizebox{\linewidth}{!}{%
\begin{tabular}{llrrrrrrrrr}
\toprule
\textbf{Group} & \textbf{Model} & \textbf{LAMBADA} & \textbf{HellaSwag} & \textbf{PIQA} & \textbf{WinoGrande} & \textbf{ARC-E} & \textbf{ARC-C} & \textbf{OBQA} & \textbf{BoolQ} & \textbf{Mean} \\
\midrule
Matched & NoPE & 30.16 & \textbf{37.88} & 66.27 & \textbf{51.93} & \textbf{50.53} & \textbf{32.44} & 31.60 & 60.37 & \textbf{45.15} \\
Matched & RoPE & \textbf{30.25} & 37.19 & 66.76 & 50.67 & 49.12 & 29.43 & \textbf{32.40} & \textbf{60.95} & 44.60 \\
Matched & Polar & 28.47 & 36.19 & \textbf{66.87} & 51.62 & 49.12 & 26.09 & 31.80 & 56.21 & 43.29 \\
\midrule
Raven/AdamW & Atma-Raven-Titans & 26.00 & \textbf{34.12} & \textbf{65.56} & \textbf{52.49} & 48.07 & \textbf{27.76} & \textbf{29.00} & 61.41 & \textbf{43.05} \\
Raven/AdamW & Raven Native & \textbf{27.48} & 33.56 & 64.69 & 51.70 & \textbf{49.12} & 27.09 & 28.80 & \textbf{61.93} & \textbf{43.05} \\
\bottomrule
\end{tabular}}
\caption{Task-level zero-shot accuracy (\%) on the eight short-context controls (22,939 examples per model). LAMBADA, WinoGrande, and BoolQ use standard accuracy; HellaSwag, PIQA, ARC-Easy, ARC-Challenge, and OpenBookQA use length-normalized accuracy. Mean is the unweighted average of the eight displayed metrics; bold is best within each comparison group.}
\label{table:base_tasks_breakdown}
\end{table}

\begin{figure}[h]
\centering
\includegraphics[width=\linewidth]{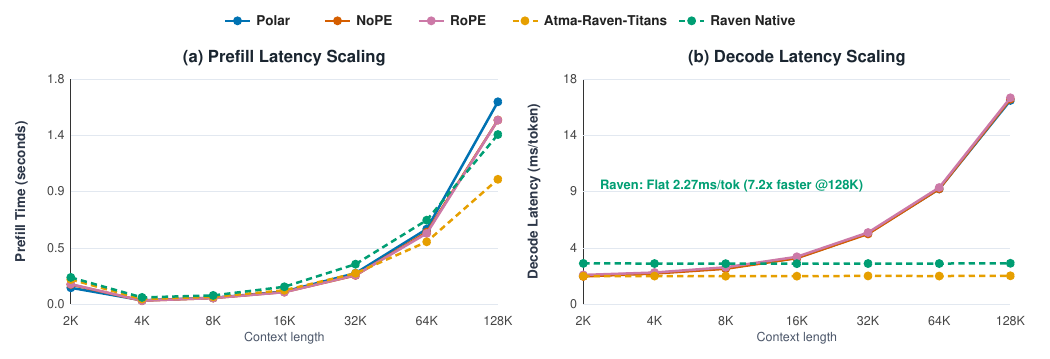}
\caption{Single-sequence serving on one L40S. Attention decode scales with KV-cache length; recurrent Raven uses fixed state.}
\label{fig:serving_scaling}
\end{figure}

\begin{table}[h]
\centering
\small
\resizebox{\linewidth}{!}{%
\begin{tabular}{llrrrrr}
\toprule
\textbf{Group} & \textbf{Model} & \textbf{Base mean} & \textbf{128K prefill} & \textbf{Decode} & \textbf{Peak alloc.} & \textbf{Train MFU} \\
& & (\%) & (s) & (ms/token) & (GiB) & (\%, 2K) \\
\midrule
Matched & NoPE & 45.1 & 1.47 & 16.4 & 39.40 & 42.77 \\
Matched & Polar & 43.3 & 1.62 & 16.3 & 39.41 & 40.44 \\
Matched & RoPE & 44.6 & 1.47 & 16.5 & 39.41 & 42.45 \\
Raven/AdamW & Atma-Raven-Titans & 43.1 & 1.00 & 2.27 & 8.62 & 30.92 \\
Raven/AdamW & Raven Native & 43.0 & 1.36 & 3.27 & 6.46 & 21.53 \\
\bottomrule
\end{tabular}}
\caption{Short-context quality and serving trade-offs. Serving uses one sequence and one timing sample; it is descriptive rather than an uncertainty estimate.}
\label{table:quality_systems_summary}
\end{table}

Table~\ref{table:base_tasks_breakdown} gives the exact per-task controls shown in main-text Figure~\ref{fig:short_context_controls}. Figure~\ref{fig:serving_scaling} shows the serving-length trend, and Table~\ref{table:quality_systems_summary} summarizes the quality--systems trade-off.

\section{Systems Implementation and Optimization Breakdown}
\label{appendix:systems}

The inference stack uses separate kernels for full-context prefill, paged autoregressive decode, and recurrent-state updates. Table~\ref{table:kernel_breakdown} distinguishes what each optimization changes. In particular, streaming removes quadratic score storage but not quadratic prefill arithmetic, while the recurrent Raven path changes the state representation itself and therefore has different scaling.

\begin{table}[!ht]
\centering
\small
\begin{tabular}{p{0.20\linewidth}p{0.34\linewidth}p{0.34\linewidth}}
\toprule
\textbf{Path} & \textbf{Optimization} & \textbf{Operational effect} \\
\midrule
Polar prefill & Block-streamed $K,V$ tiles with online $(M,L,Q_2,S)$ accumulation & Avoids materializing $T\times T$ scores; retains full-context $O(T^2)$ arithmetic. \\
Polar decode & Paged KV-cache kernel grouping the four query heads that share each KV head & Reuses KV loads across the GQA group; decode remains $O(T)$ in cache length. \\
Gated-delta step & Fused in-place decay, prediction, delta write, and readout & Avoids materializing per-step intermediate states and preserves $O(1)$ recurrent state size. \\
Ragged prefill & Packed prompts with grouped compatible tile shapes & Reduces padding work on heterogeneous prompt batches. \\
Compiler boundary & Recurrent scan and custom kernels exposed as opaque operations & Prevents compiler graph breaks around the stateful update while retaining explicit numerical tests. \\
\bottomrule
\end{tabular}
\caption{Implementation-level optimization breakdown. Complexity statements describe the sequence-mixing path and exclude shared MLP and embedding costs.}
\label{table:kernel_breakdown}
\end{table}

\begin{landscape}
\subsection{Single-sequence scaling}
Table~\ref{table:serving_detailed_matrix} gives the underlying single-sequence measurements. On one NVIDIA L40S, attention decode increases from approximately 2.3 ms/token at 2K to 16.3--16.5 ms/token at 128K because every step reads a longer KV cache. The recurrent models remain nearly flat because they update fixed-size state. Prefill is less monotonic at short lengths because launch overhead and kernel occupancy dominate before the matrix operations reach their efficient regime. These timings use batch size 1, 32 generated tokens, and one recorded timing sample per cell. They are descriptive measurements, not confidence intervals.
\begingroup
\small
\setlength{\tabcolsep}{4pt}
\begin{longtable}{llrrrrrrr}
\caption{Detailed prefill and decode scaling from 2K to 128K tokens on one NVIDIA L40S GPU (batch size 1, 32 generated tokens).}\label{table:serving_detailed_matrix}\\
\toprule
\textbf{Model} & \textbf{Metric} & \textbf{2K} & \textbf{4K} & \textbf{8K} & \textbf{16K} & \textbf{32K} & \textbf{64K} & \textbf{128K} \\
\midrule
\endfirsthead
\caption[]{Detailed prefill and decode scaling (continued).}\\
\toprule
\textbf{Model} & \textbf{Metric} & \textbf{2K} & \textbf{4K} & \textbf{8K} & \textbf{16K} & \textbf{32K} & \textbf{64K} & \textbf{128K} \\
\midrule
\endhead
\midrule
\multicolumn{9}{r}{Continued on next page}\\
\endfoot
\bottomrule
\endlastfoot
Polar & Prefill latency (s) & 0.131 & 0.031 & 0.051 & 0.101 & 0.248 & 0.601 & 1.619 \\
& Prefill throughput (tok/s) & 15,640 & 133,884 & 160,145 & 161,551 & 132,091 & 109,050 & 80,959 \\
& Decode latency (ms/tok) & 2.27 & 2.49 & 2.88 & 3.70 & 5.66 & 9.22 & 16.30 \\
& Decode throughput (tok/s) & 441.2 & 402.3 & 346.6 & 270.1 & 176.8 & 108.5 & 61.3 \\
\midrule
NoPE & Prefill latency (s) & 0.157 & 0.028 & 0.049 & 0.097 & 0.230 & 0.576 & 1.473 \\
& Prefill throughput (tok/s) & 13,070 & 146,568 & 166,504 & 168,510 & 142,319 & 113,810 & 89,013 \\
& Decode latency (ms/tok) & 2.22 & 2.45 & 2.83 & 3.68 & 5.62 & 9.22 & 16.40 \\
& Decode throughput (tok/s) & 450.1 & 408.8 & 353.0 & 271.7 & 177.8 & 108.4 & 61.0 \\
\midrule
RoPE & Prefill latency (s) & 0.160 & 0.028 & 0.049 & 0.096 & 0.232 & 0.566 & 1.472 \\
& Prefill throughput (tok/s) & 12,761 & 147,089 & 168,287 & 170,936 & 141,249 & 115,786 & 89,015 \\
& Decode latency (ms/tok) & 2.34 & 2.53 & 2.95 & 3.79 & 5.73 & 9.33 & 16.51 \\
& Decode throughput (tok/s) & 427.4 & 395.3 & 338.7 & 263.7 & 174.5 & 107.2 & 60.6 \\
\midrule
Atma-Raven-Titans & Prefill latency (s) & 0.198 & 0.041 & 0.059 & 0.111 & 0.247 & 0.499 & 0.999 \\
& Prefill throughput (tok/s) & 10,334 & 100,989 & 137,903 & 147,434 & 132,916 & 131,464 & 131,195 \\
& Decode latency (ms/tok) & 2.24 & 2.25 & 2.24 & 2.24 & 2.26 & 2.25 & 2.27 \\
& Decode throughput (tok/s) & 446.6 & 445.1 & 446.0 & 447.1 & 442.6 & 444.8 & 440.3 \\
\midrule
Raven Native & Prefill latency (s) & 0.214 & 0.054 & 0.070 & 0.139 & 0.320 & 0.672 & 1.357 \\
& Prefill throughput (tok/s) & 9,558 & 75,738 & 116,378 & 117,605 & 102,354 & 97,527 & 96,559 \\
& Decode latency (ms/tok) & 3.26 & 3.25 & 3.24 & 3.24 & 3.25 & 3.25 & 3.27 \\
& Decode throughput (tok/s) & 306.4 & 307.5 & 308.2 & 308.3 & 307.6 & 307.5 & 306.1 \\
\end{longtable}
\endgroup
\end{landscape}

\subsection{Production-framework stress comparison}
To test whether the custom kernels introduce an obvious systems bottleneck at a larger shape, we configured ATMA's inference engine with a 9.2B-parameter stress shape and compared it with vLLM 0.25.0 serving Qwen3.5-9B on the same L40S in BF16~\citep{kwon2023vllm}. The models are not weight- or architecture-matched, so Table~\ref{table:vllm_comparison_matrix} is a backend stress comparison rather than a model-quality or universal serving claim. ATMA trails on short homogeneous prefill, is similar on long heterogeneous prefill, and leads the recorded high-batch decode cells.

\begin{table}[!ht]
\centering
\small
\setlength{\tabcolsep}{4pt}
\resizebox{\linewidth}{!}{%
\begin{tabular}{llrrr}
\toprule
\textbf{Phase} & \textbf{Batch/workload} & \textbf{ATMA engine} & \textbf{vLLM/Qwen3.5-9B} & \textbf{Ratio} \\
\midrule
Prefill & Homogeneous $B=8,T=512$ (4,096 tok) & 11,947 tok/s & 12,772 tok/s & 0.94$\times$ \\
& Heterogeneous short-heavy (816 tok) & 5,752 tok/s & 8,002 tok/s & 0.72$\times$ \\
& Heterogeneous mixed (4,384 tok) & 12,779 tok/s & 11,996 tok/s & 1.07$\times$ \\
& Heterogeneous long-tail (8,192 tok) & 12,251 tok/s & 11,721 tok/s & 1.05$\times$ \\
\midrule
Decode, context 512 & Batch 128 & 2,939 tok/s & 2,247 tok/s & 1.31$\times$ \\
& Batch 256 & 4,501 tok/s & 2,921 tok/s & 1.54$\times$ \\
& Batch 512 & 5,919 tok/s & 2,558 tok/s & 2.31$\times$ \\
& Heterogeneous $B=128$ (context 64--1536) & 2,950 tok/s & 2,403 tok/s & 1.23$\times$ \\
\bottomrule
\end{tabular}}
\caption{Recorded systems-scale throughput for ATMA's 9.2B stress shape and vLLM 0.25.0 serving Qwen3.5-9B on one L40S (BF16, greedy sampling). Ratio is ATMA/vLLM.}
\label{table:vllm_comparison_matrix}
\end{table}

\section{Stress and Environment Diagnostics}
\label{appendix:diagnostics}

\begin{figure}[h]
\centering
\includegraphics[width=\linewidth]{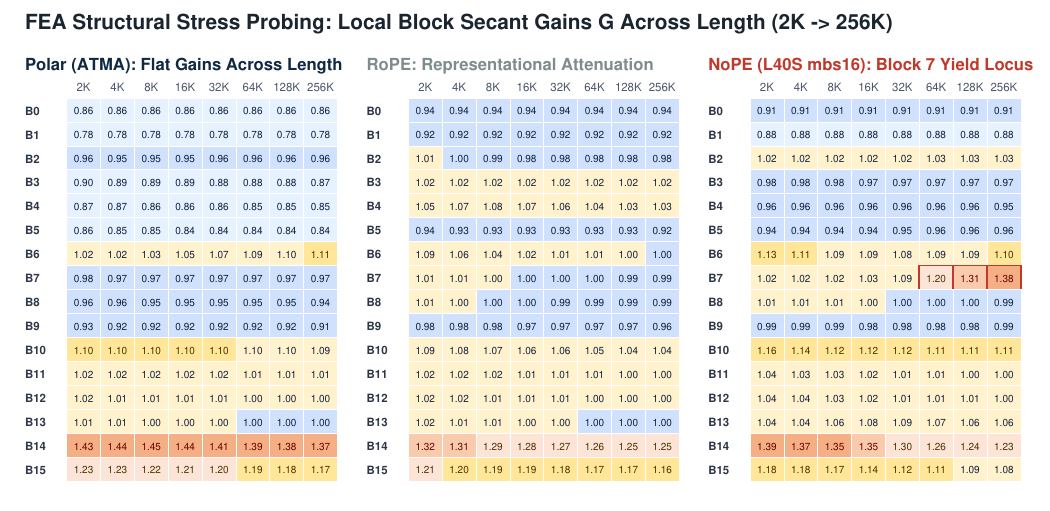}
\caption{Block-wise sampled secant gains over 64 nested FinePDFs documents and 64 perturbation directions per document. The probe localizes checkpoint-specific changes but does not estimate a global Lipschitz constant.}
\label{fig:appendix_stress_heatmap}
\end{figure}

Figure~\ref{fig:appendix_stress_heatmap} reports the sampled gains. For each block $f_i$, the diagnostic samples
\begin{equation}
G_i=\frac{\lVert f_i(x+\delta x)-f_i(x)\rVert_{\rm rms}}{\lVert\delta x\rVert_{\rm rms}},
\qquad \lVert\delta x\rVert_{\rm rms}=0.02\lVert x\rVert_{\rm rms}.
\end{equation}
The maximum sampled gain is a lower bound on the local operator norm along sampled directions. It is not a tight bound on the block Lipschitz constant. In the observed NoPE collapse, changes first concentrate near blocks 6--7 and later propagate through the residual stream. Polar's sampled gains and activation moments remain comparatively flat through 256K.

\section{Checkpoint-Variability Diagnostic Journey}
\label{appendix:environment}

The variability finding emerged sequentially. Stage I used L4 GPUs, followed by the first 10B NoPE control on an L4. The anomaly became visible only after the remaining training moved to L40S: short-context convergence remained nearly unchanged while the 256K extrapolation results diverged. Table~\ref{table:environment_diagnostic_journey} records the subsequent elimination sequence. The recorded run configurations do not provide a paired controlled seed, so random initialization and the resulting stochastic trajectory remain confounded with the infrastructure transition.

\begin{table}[!ht]
\centering
\small
\setlength{\tabcolsep}{4pt}
\begin{tabular}{>{\raggedright\arraybackslash}p{0.18\linewidth}>{\raggedright\arraybackslash}p{0.34\linewidth}>{\raggedright\arraybackslash}p{0.36\linewidth}}
\toprule
\textbf{Step} & \textbf{Observation or intervention} & \textbf{Diagnostic implication} \\
\midrule
L4 baseline & Stage I and the first 10B NoPE control used L4, microbatch 4; 256K NLL was 4.07. & Established the original extrapolation behavior, but not a replicated device-class estimate. \\
L40S transition & The promoted NoPE run used L40S, microbatch 16; 2K validation ended at 2.8126 versus 2.8160 on L4, but 256K NLL rose to 10.77. & Revealed that nearly coincident short-context curves did not predict the long-context endpoint. \\
Evaluation replay & L4 and L40S checkpoints were re-evaluated under PyTorch 2.12 and 2.13 with unchanged scores. & Excluded evaluation-version mismatch for these checkpoints. \\
Microbatch control & NoPE was retrained on L40S with microbatch 4; 2K validation ended at 2.8128 and 256K NLL at 9.90. & Excluded microbatch size alone as the explanation in this comparison. \\
Deterministic fixtures & Synthetic 2K stream fixtures across the tested versions and microbatches showed no gradient or loss divergence. & Found no short-context execution discrepancy; did not test the full stochastic training trajectory. \\
\bottomrule
\end{tabular}
\caption{Diagnostic journey from the L4 control to the L40S follow-ups. Each implication is limited to the intervention actually performed.}
\label{table:environment_diagnostic_journey}
\end{table}

All five headline models were trained in the same L40S software and hardware environment, so their primary comparison is internally matched. The earlier L4 checkpoint is retained only as a post-hoc case study. It motivates a broader audit principle: long-context extrapolation claims should report device class, kernels, seeds, and extended-context diagnostics rather than infer robustness from short-context loss alone. The present evidence does not identify CUDA reduction order, floating-point accumulation, device class, or any other single mechanism as causal because the L4/L40S runs were neither initialized from a recorded common state nor replicated across seeds and kernel configurations.

\end{document}